\documentclass{article}

\usepackage[preprint]{neurips_2026}

\usepackage[utf8]{inputenc} % allow utf-8 input
\usepackage[T1]{fontenc}    % use 8-bit T1 fonts
\usepackage{hyperref}       % hyperlinks
\usepackage{url}            % simple URL typesetting
\usepackage{booktabs}       % professional-quality tables
\usepackage{amsfonts}       % blackboard math symbols
\usepackage{nicefrac}       % compact symbols for 1/2, etc.
\usepackage{microtype}      % microtypography
\usepackage{xcolor}         % colors

\usepackage{graphicx}
\input{macros}

\usepackage{algorithm}
\usepackage{algorithmic}
\usepackage{enumitem}
\usepackage{listings}
\usepackage{caption}

\title{Delightful Gradients Accelerate Corner Escape}

\author{%
  Jincheng Mei \\
  Google DeepMind\\
  \texttt{jcmei@google.com} \\
  \And
  Ian Osband \\
  Google DeepMind\\
  \texttt{iosband@google.com}  \\
}

\begin{document}

\maketitle

\begin{abstract}
Softmax policy gradient converges at $O(1/t)$, but its transient behavior near sub-optimal corners of the simplex can be exponentially slow.
The bottleneck is self-trapping: negative-advantage actions reinforce the corner policy and can initially push the optimal action backward.
We study \emph{Delightful Policy Gradient} (DG), which gates each policy-gradient term by the product of advantage and action surprisal.
For $K$-armed bandits, we prove that the zero-temperature limit of DG removes this corner-trapping mechanism on a quantitative sector near any sub-optimal corner, yielding a first-exit escape bound logarithmic in the initial probability ratio.
At every fixed temperature, the same local mechanism persists because harmful actions are polynomially suppressed as they become rare.
A key structural insight is that every action better than the corner action is an \emph{ally}: its contribution to escape is non-negative.
Combining corner instability with a monotonic value improvement identity, we prove that DG converges globally to the optimal policy in both bandits and tabular MDPs at an asymptotic $O(1/t)$ rate.
We also show, via an exact counterexample, that this tabular mechanism can fail under shared function approximation.
In MNIST contextual bandits with a shared-parameter neural network, DG nevertheless recovers from bad initializations faster than standard policy gradient, suggesting that the counterexample marks a boundary of the theory rather than a practical prohibition.
\end{abstract}

%%%%%%%%%%%%%%%%%%%%%%%%%%%%%%%%%%%%%%%%%%%%%%%%%%%%%%%%%%%%%%%%%%%%%%%%%%%%%%%%
%%%%%%%%%%%%%%%%%%%%%%%%%%%%%%%%%%%%%%%%%%%%%%%%%%%%%%%%%%%%%%%%%%%%%%%%%%%%%%%%
\section{Introduction}
\label{sec:intro}

Policy gradient is central to modern reinforcement learning~\citep{sutton1999policy} and underpins large-language-model alignment~\citep{ouyang2022training}.
With a softmax parameterization, policy gradient guarantees $O(1/t)$ convergence to globally optimal policies~\citep{mei2020convergence, agarwal2021theory}.
However, this guarantee hides a serious transient pathology: even with exact gradients, softmax policy gradient can take exponential time to escape sub-optimal corners~\citep{li2023softmax}.
Can a simple modification to the gradient update remove this corner trap, without changing the policy class or the optimization landscape?

The mechanism is self-trapping.
Near a sub-optimal corner where $\pi(j) \approx 1$, the softmax gradient flow for any arm $a$ satisfies $\dot\theta(a) = \pi(a)\,U(a)$.
The optimal arm's logit always increases, but its growth rate $\pi(1)\,U(1)$ is throttled by $\pi(1) \approx 0$.
Meanwhile, arms worse than $j$ pull the mean reward below $r(j)$, giving the corner arm a positive advantage; since $\pi(j) \approx 1$, its logit grows at nearly full scale.
The corner arm therefore gains on the optimal arm, creating a region of the simplex where the sub-optimal arm's log-probability grows faster than the optimal arm's~\citep{mei2020convergence}.
The result is a self-reinforcing trap: the more concentrated the policy becomes on the wrong arm, the harder it is to escape.

Entropy regularization addresses this by changing the optimization landscape, yielding linear convergence to a biased optimum~\citep{mei2020convergence, cen2022fast, lan2023policy, ahmed2019entropy}.
Natural policy gradient and policy mirror descent improve the geometry through Fisher preconditioning~\citep{kakade2001natural, zhan2023accelerating}, but the corner pathology persists for methods that follow the vanilla gradient direction~\citep{li2023softmax}.

The \emph{Delightful Policy Gradient} (DG) takes a different route~\citep{osband2026delightfulpolicygradient, osband2026delightfuldistributedpolicygradient, osband2026doesgradientsparkjoy}.
DG gates each policy-gradient term by \emph{delight}, the product of advantage and action surprisal.
In the zero-temperature limit, this removes every negative-advantage term entirely.
At any fixed temperature, the surprisal factor still suppresses harmful rare actions polynomially as they become rarer.
A structural insight of the analysis is that every action better than the corner action is an \emph{ally}: its contribution to escape is non-negative, so allies help the optimal arm escape rather than competing with it.

We formalize this picture for general $K$-armed bandits and tabular MDPs.
First, on a quantitative sector near any sub-optimal corner, the zero-temperature limit of DG ensures that the optimal arm's logit grows faster than the corner arm's, and finite-temperature DG preserves the same local sector gap (Theorems~\ref{thm:corner_escape}--\ref{thm:sigmoid_escape}).
Second, in the zero-temperature limit this yields a first-exit escape bound logarithmic in the initial probability ratio, replacing PG's potentially exponential corner transient with a logarithmic one (Theorem~\ref{thm:escape_time}). 
Third, we prove that EG converges globally to the optimal policy in both bandits
and tabular MDPs: monotonic value improvement, a Bellman-optimality characterization
of limit points, and corner instability combine to give global convergence at
an asymptotic $O(1/t)$ rate
(Theorems~\ref{thm:global_conv}--\ref{thm:mdp_convergence},
Corollaries~\ref{cor:rate}--\ref{cor:mdp_rate}).
Notably, because sub-optimal corners are repellers under EG/DG rather than positive-measure attractors as under PG, the global convergence proofs are substantially simpler, no indirect escape argument is needed (Remark~\ref{rem:proof_comparison}).
Fourth, the surprisal factor is essential for this rarity-sensitive suppression: gating by advantage alone assigns harmful actions a non-vanishing weight even when they are extremely rare (Remark~\ref{rem:surprisal}).
Finally, we show that this mechanism need not survive shared function approximation.
With shared parameters and conflicting advantages across states, DG can admit a sub-optimal interior fixed point, and we give an exact counterexample (Section~\ref{sec:counterexample}).
On MNIST contextual bandits with a shared-parameter MLP, however, DG still recovers from bad initializations faster than PG, suggesting the counterexample identifies a boundary of the theory rather than a practical prohibition (Section~\ref{sec:experiments}).

\paragraph{Relation to prior work.}
\citet{mei2020convergence} prove $O(1/t)$ convergence for softmax PG, with constants that can be exponentially small near corners, and \citet{li2023softmax} make this lower bound precise.
Closest in spirit is \citet{mei2020gravitational}, who propose \emph{escort policy gradient} by changing the policy transform itself.
We keep the softmax policy class fixed and instead modify the gradient update through delight gating.
The contribution is a mechanism-level analysis of how DG removes the corner trap: zero temperature eliminates harmful terms entirely, finite temperature suppresses them polynomially, and the resulting sector gap yields a logarithmic first-exit bound.

%%%%%%%%%%%%%%%%%%%%%%%%%%%%%%%%%%%%%%%%%%%%%%%%%%%%%%%%%%%%%%%%%%%%%%%%%%%%%%%%
\section{Settings and Algorithms}
\label{sec:algorithms}

We consider a finite Markov Decision Process (MDP) defined by the tuple $(\mathcal{S}, \mathcal{A}, P, r, \gamma, \rho)$. Here, $\mathcal{S}$ and $\mathcal{A}$ are finite state and action spaces, $P(s'|s, a)$ is the transition probability from state $s$ to $s'$ under action $a$, $r(s, a) \in [0, 1]$ is the reward function, $\gamma \in [0, 1)$ is the discount factor, and $\rho$ is the initial state distribution. The agent follows a stochastic policy $\pi(a|s)$, which we parameterize using the softmax function,
\begin{equation}
\pi_{\theta}(a|s) = \frac{\exp(\theta(s, a))}{\sum_{a' \in \mathcal{A}} \exp(\theta(s, a'))},
\end{equation}
where $\theta \in \mathbb{R}^{|\mathcal{S}| \times |\mathcal{A}|}$. For a given policy $\pi$, the state-value function $V^{\pi}(s)$ and action-value function $Q^{\pi}(s, a)$ are defined as the expected discounted return,
\begin{equation*}
V^{\pi}(s) = \mathbb{E}_{\pi} \left[ \sum_{t=0}^{\infty} \gamma^t r(s_t, a_t) \mid s_0 = s \right], \quad Q^{\pi}(s, a) = \mathbb{E}_{\pi} \left[ \sum_{t=0}^{\infty} \gamma^t r(s_t, a_t) \mid s_0 = s, a_0 = a \right],
\end{equation*}
The goal is to maximize the expected discounted reward $V(\pi) = \mathbb{E}_{s \sim \rho} [V^{\pi}(s)]$. We denote the discounted state visitation distribution as $d_{\rho}^{\pi}(s) = (1-\gamma) \sum_{t=0}^{\infty} \gamma^t \text{Pr}(s_t = s | \rho, \pi)$.

For illustration of intuition purpose, we will consider the simplest bandit case.

\begin{definition}[$K$-armed bandit]
\label{assump:bandit}
$K \ge 2$ arms with reward vector $r \in [0,1]^K$.
Unique optimal arm $a^* = \argmax_a r(a)$.
The policy is parameterized by logits $\theta \in \R^K$ through the softmax map $\pi_\theta(a) = e^{\theta(a)} / \sum_{a'} e^{\theta(a')}$.
Without loss of generality (by relabeling arms), we order rewards $r(1) > r(2) > \cdots > r(K)$, so $a^* = 1$.
\end{definition}

\begin{minipage}[t]{0.49\textwidth} % Added [t]
    \vspace{-15pt} % Forces baseline to the top
    \begin{algorithm}[H]
    \caption{Softmax Policy Gradient (PG)}
    \label{alg:pg_mdp}
    \begin{algorithmic}
        \STATE {\bf Input:} Learning rate $\alpha > 0$.
        \STATE {\bf Initialize:} $\theta_1(s, a)$ for all $s, a$.
        \WHILE{$t \ge 1$}
            \FOR{each state $s \in \mathcal{S}$}
                \FOR{each action $a \in \mathcal{A}$}
                    \STATE $U(s,a) \leftarrow Q^{\pi_t}(s, a) - V^{\pi_t}(s)$.
                \ENDFOR
                \STATE \scalebox{0.9}{$\Delta \theta_t(s, \cdot) \leftarrow \sum\limits_{a \in \mathcal{A}} \pi_{\theta_t}(a|s) \;\nabla_\theta \log \pi_{\theta_t}(a|s)\; U(s,a)$}.
            \ENDFOR
            \STATE $\theta_{t+1} \leftarrow \theta_t + \alpha\, \Delta \theta_t$.
        \ENDWHILE
    \end{algorithmic}
    \end{algorithm}
\end{minipage}
\hfill
\begin{minipage}[t]{0.48\textwidth} % Added [t]
    \vspace{-15pt} % Forces baseline to the top
    \begin{algorithm}[H]
    \caption{Delightful Policy Gradient (DG)}
    \label{alg:dg_mdp}
    \begin{algorithmic}
        \STATE {\bf Input:} Stepsize $\alpha > 0$, temperature $\eta > 0$.
        \STATE {\bf Initialize:} $\theta_1(s, a)$ for all $s, a$.
        \WHILE{$t \ge 1$}
            \FOR{each state $s \in \mathcal{S}$}
                \FOR{each action $a \in \mathcal{A}$}
                    \STATE $U(s,a) \leftarrow Q^{\pi_t}(s, a) - V^{\pi_t}(s)$.
                    \STATE $\ell(s,a) \leftarrow -\log \pi_{\theta_t}(a|s)$.
                    \STATE $w(s,a) \leftarrow \sigma\!\bigl(U(s,a)\;\ell(s,a)\,/\,\eta\bigr)$.
                \ENDFOR
                \STATE \scalebox{0.9}{$\Delta \theta_t(s, \cdot) \leftarrow \sum\limits_{a \in \mathcal{A}} w(s,a)\;\pi_{\theta_t}(a|s) \;\nabla_\theta \log \pi_{\theta_t}(a|s)\; U(s,a)$}.
            \ENDFOR
            \STATE $\theta_{t+1} \leftarrow \theta_t + \alpha\, \Delta \theta_t$.
        \ENDWHILE
    \end{algorithmic}
    \end{algorithm}
\end{minipage}

% \subsection{Standard Softmax Policy Gradient}
According to the policy gradient theorem~\citep{sutton1999policy}, the gradient of the objective $V(\pi_{\theta})$ with respect to the parameters $\theta$ is given by
\begin{equation}
\nabla_{\theta} V(\pi_{\theta}) = \frac{1}{1-\gamma} \sum_{s \in \mathcal{S}} d_{\rho}^{\pi}(s) \sum_{a \in \mathcal{A}} \nabla_{\theta} \pi_{\theta}(a|s) Q^{\pi}(s, a).
\end{equation}
Using the score function identity $\nabla_{\theta} \pi_{\theta} = \pi_{\theta} \nabla_{\theta} \log \pi_{\theta}$ and defining the advantage $U(s,a) = Q^{\pi}(s, a) - V^{\pi}(s)$, the standard update rule is summarized in Algorithm~\ref{alg:pg_mdp}. In the bandit case, the advantage of action $a$ is $U(a) := r(a) - \pi_\theta^\top r$.
We write $\Delta_{ab} := r(a) - r(b)$ for pairwise gaps and $\Delta := \Delta_{12} > 0$ for the optimality gap.

% \subsection{Delightful Policy Gradient (DG)}

Delightful Policy Gradient (DG) modifies the update by gating each action's contribution according to \emph{delight}~\citep{osband2026delightfulpolicygradient}.
The \emph{surprisal} $\ell(a) := -\log \pi(a)$ measures how unexpected action $a$ is.
The \emph{delight} is the product $U(a)\,\ell(a)$.
The DG gate is
\[
w(a) := \sigma\!\bigl(U(a)\,\ell(a)\,/\,\eta\bigr)
\]
for temperature $\eta > 0$, giving the gated drift:
\begin{equation}
\label{eq:dg_drift}
\dot\theta(a)
=
\sum_{a'=1}^{K} w(a')\,\pi(a')\,U(a')\,(\Ind\{a=a'\} - \pi(a)).
\end{equation}
In the limit $\eta \to 0$, the gate becomes a hard threshold: $w(a) \to \Ind\{U(a) > 0\}$.
This removes every negative-advantage term entirely.
We call the result the \emph{enlightened gradient} (EG): it retains every delightful action, no matter how rare, and ignores the rest.

%%%%%%%%%%%%%%%%%%%%%%%%%%%%%%%%%%%%%%%%%%%%%%%%%%%%%%%%%%%%%%%%%%%%%%%%%%%%%%%%
\section{Geometric Illustration in Bandits}
\label{sec:bandit_convergence}

The behavior of DG is best illustrated in the $K$-armed bandit simplex. Standard PG exhibits ``sticky'' sub-optimal corners where $\pi(j) \approx 1$, leading to exponential escape times~\citep{mei2020convergence,li2023softmax}. Figure~\ref{fig:adversarial_initialization} shows how DG modifies the gradient flow to eliminate these bad regions.

\begin{figure*}[h] % Use [t] for top of page alignment in figure*
\centering
\begin{subfigure}[b]{.27\linewidth}
    \includegraphics[width=\linewidth]{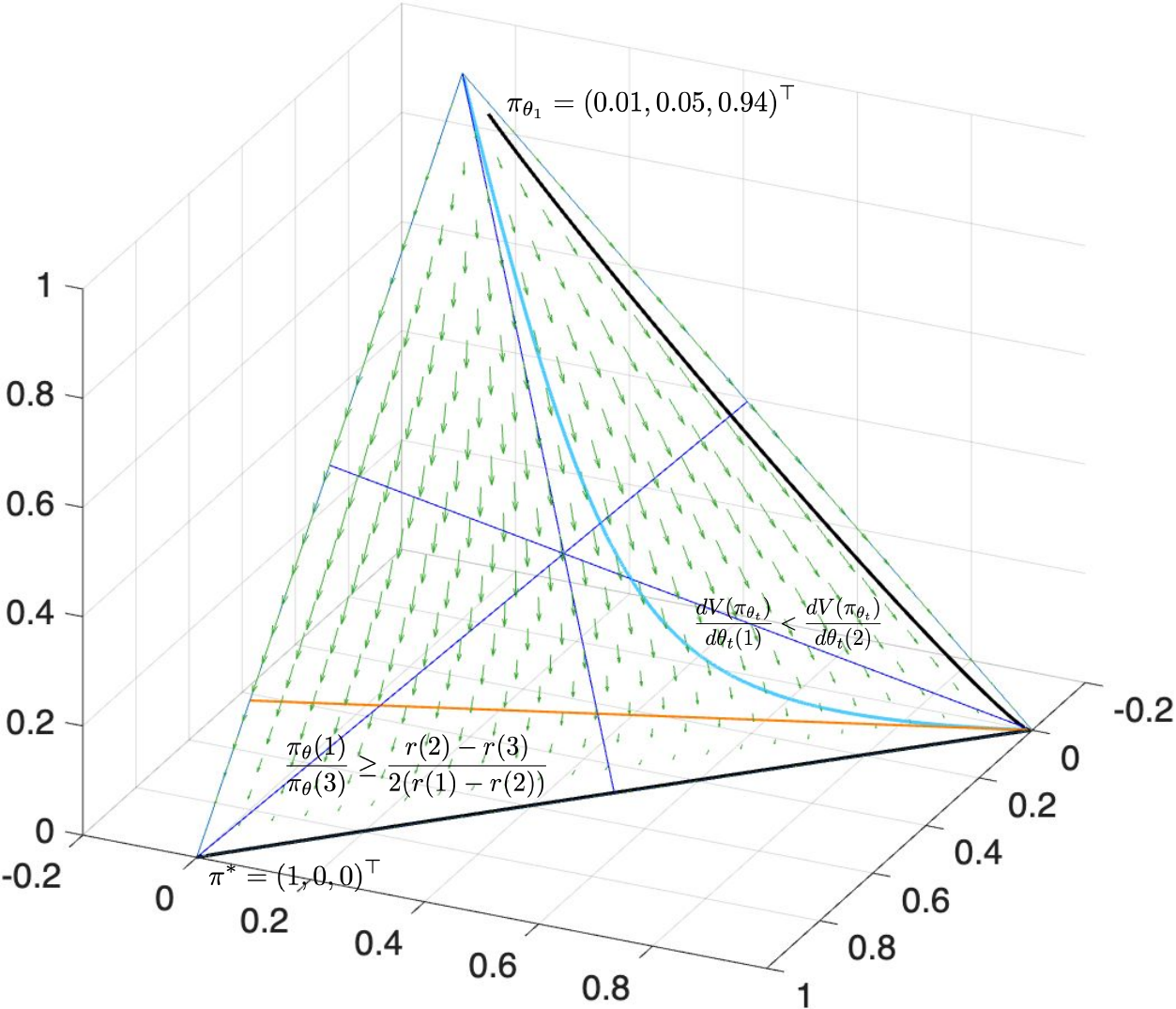}
    \caption{PG gradient flow.}\label{fig:fig1a}
\end{subfigure}\hfill % Use \hfill to distribute space evenly
\begin{subfigure}[b]{.27\linewidth}
    \includegraphics[width=\linewidth]{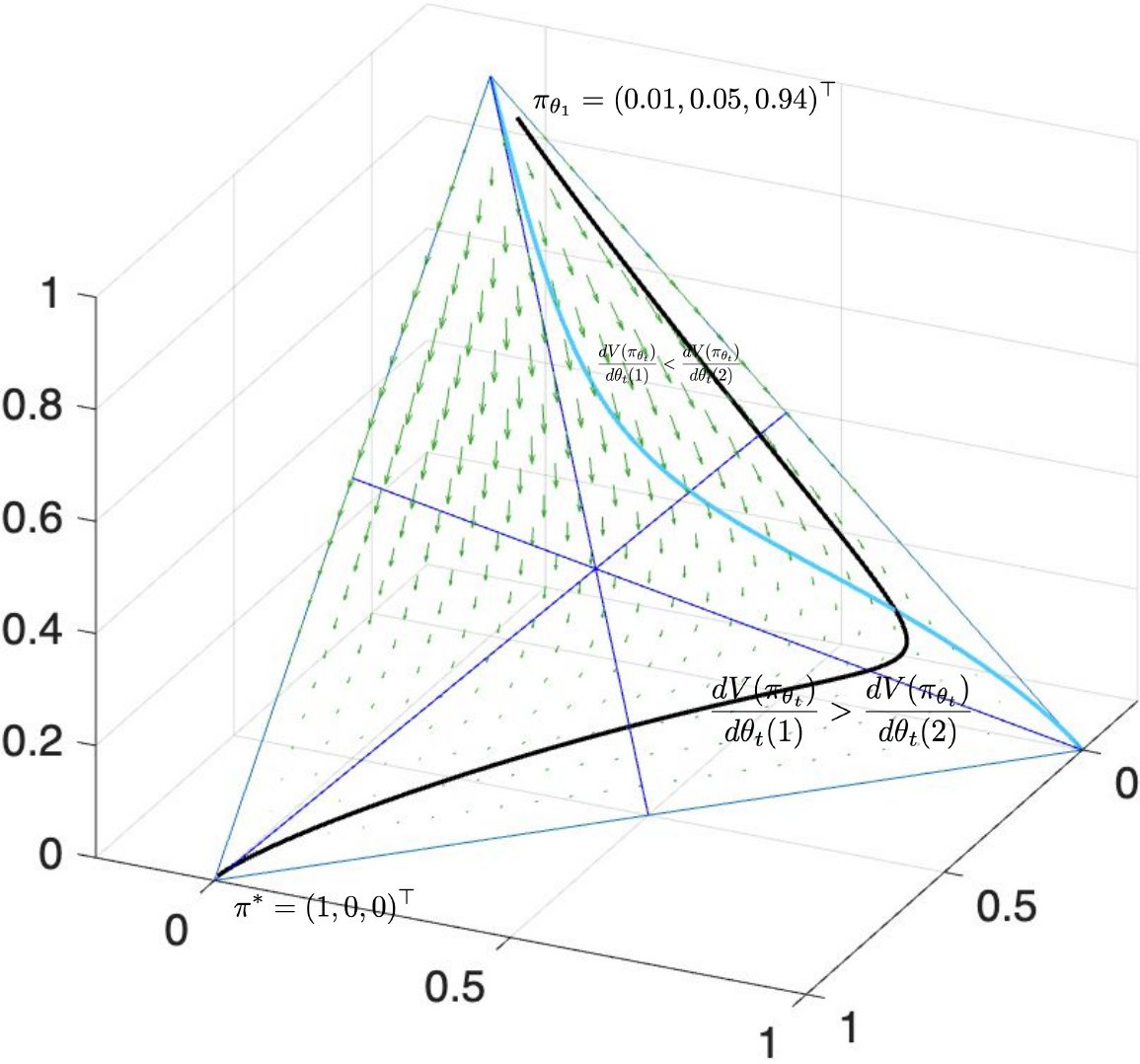}
    \caption{DG gradient flow.}\label{fig:fig1b}
\end{subfigure}\hfill
\begin{subfigure}[b]{.23\linewidth}
    \includegraphics[width=\linewidth]{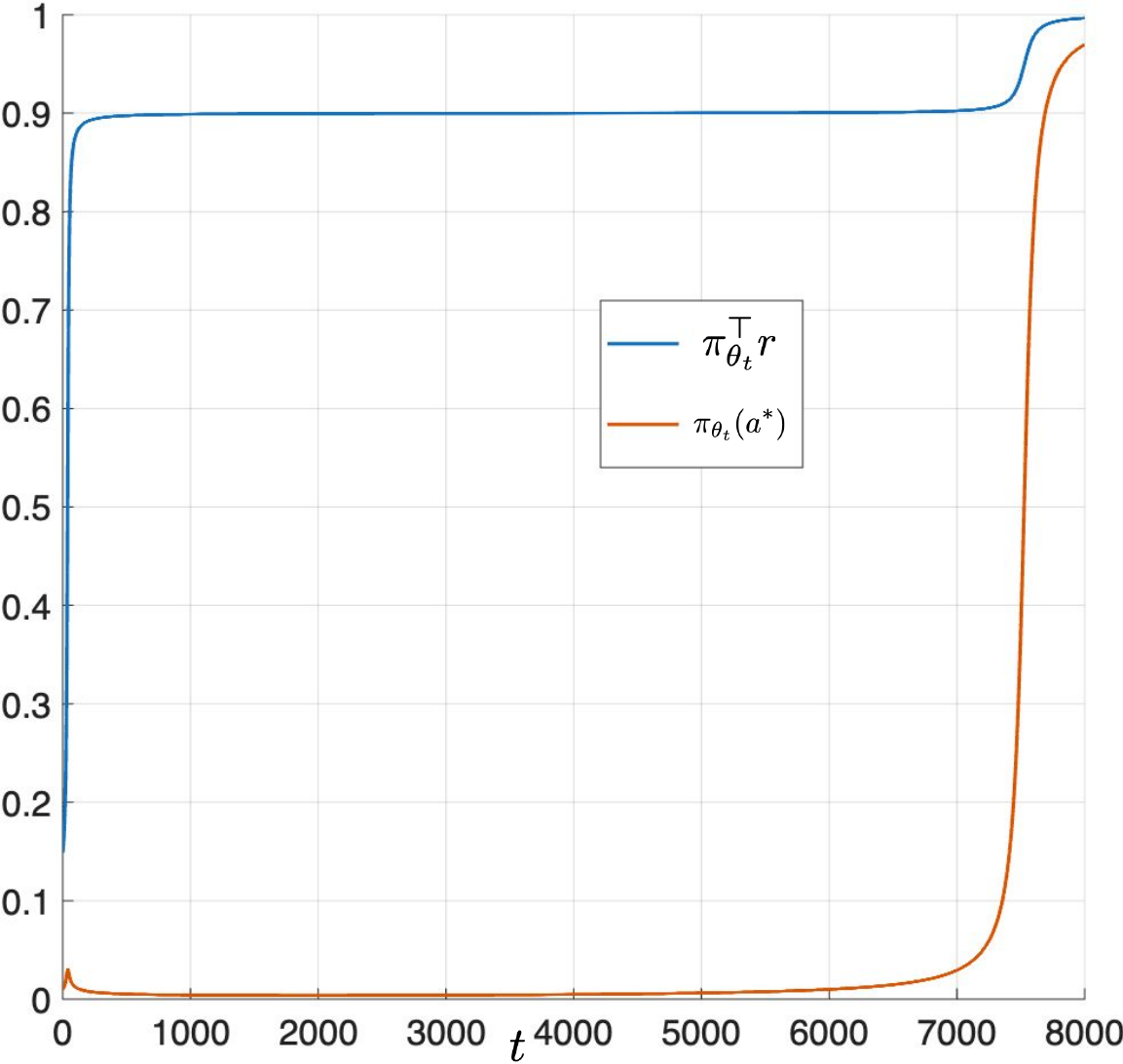}
    \caption{PG bad initialization.}\label{fig:fig1c}
\end{subfigure}\hfill
\begin{subfigure}[b]{.23\linewidth}
    \includegraphics[width=\linewidth]{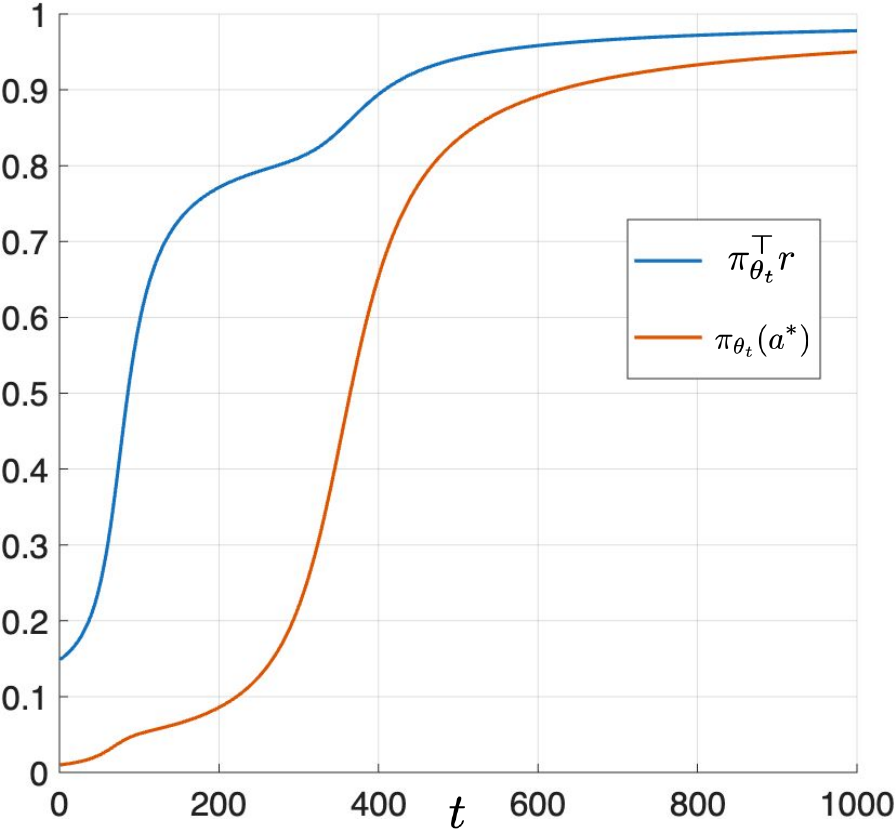}
    \caption{DG bad initialization.}\label{fig:fig1d}
\end{subfigure}
\caption{Visualization of DG gradient fields ($\eta = 1.0$) and comparison with standard PG. The light blue curves in (a) and (b) indicate where $\frac{ d V(\pi_{\theta_t}) }{d \theta_t(1)}  = \frac{ d V(\pi_{\theta_t}) }{d \theta_t(2)}$.}
\label{fig:adversarial_initialization}
\end{figure*}

\subsection{Visualization of Gradient Fields}
\label{subsec:visual}

To motivate the structural advantages of the Delightful Policy Gradient (DG), we visualize the gradient fields on the probability simplex for a 3-armed bandit with rewards $r = [1.0, 0.9, 0.1]^\top$~\citep{mei2020convergence}. We specifically examine the behavior near sub-optimal corners to identify the ``bad regions'' where the optimal action $a^*=1$ fails to receive the largest parameter update.

\paragraph{Observations from standard PG (Figure 1a).}
Standard softmax PG exhibits a ``bad region'' near sub-optimal corners~\citep{mei2020convergence}. In Figure 1a, the area below the red line is where a sub-optimal action's parameter increases faster than the optimal action's. Near the corner, the optimal arm dominates only if
\begin{equation}
\label{eq:pg_ratio_bound}
\frac{\pi_{\theta}(1)}{\pi_{\theta}(3)} \ge \frac{r(2) - r(3)}{2(r(1) - r(2))}.
\end{equation}
When the gap $r(1) - r(2)$ is small, this bad region expands.

\paragraph{Observations from DG (Figure 1b).}
Under identical rewards, the bad region disappears near corners (Figure 1b). The optimal action $a^*$ \textbf{unconditionally} receives the largest gradient update whenever the policy is concentrated near a sub-optimal corner.

\paragraph{Iteration Comparison (Figures 1c and 1d).}
From the same bad initialization $\pi_{\theta_1} = (0.01, 0.05, 0.94)^\top$, PG is attracted toward the sub-optimal corner and requires over 7000 iterations to escape. DG escapes much faster, maintaining a direct trajectory toward the optimum.

\subsection{Why DG Escapes Corners Faster}
\label{subsec:intuition}
The visualization reveals two facts: DG eliminates the bad region near
sub-optimal corners, and it escapes those corners orders of magnitude faster
than PG.
Both follow from a single mechanism, the removal of harmful arms, which we now explain before formalizing in Section~\ref{sec:corner_escape}.

Under continuous-time gradient flow, the softmax logit of each arm evolves as
\begin{equation}
\label{eq:pg_drift}
\dot\theta(a)
= \sum_{i=1}^{K} \pi(i)\,U(i)\,\bigl(\Ind\{a{=}i\} - \pi(a)\bigr).
\end{equation}
\paragraph{Self-trapping mechanism.}
Since advantages are centered ($\sum_i \pi(i)\,U(i) = 0$), this simplifies to
\begin{equation}
\label{eq:pg_simplified}
\dot\theta(a) = \pi(a)\,U(a).
\end{equation}
Each arm's logit moves in the direction of its own advantage,
at a rate proportional to its probability.
When $\pi(j) \approx 1$ for a sub-optimal arm~$j$, the corner arm's logit
grows at rate $\pi(j)\,U(j) \approx U(j)$,
while the optimal arm's logit grows at the throttled rate
$\pi(1)\,U(1) \approx \pi(1)$.
As long as $U(j) > 0$, which occurs whenever arms worse than~$j$
pull the mean reward below~$r(j)$, the corner arm gains on the
optimal arm, and the trap deepens.

\paragraph{Arms as allies and enemies.}
To understand \emph{which} arms cause the trap, fix a corner arm
$j \in \{2,\ldots,K\}$ and write $\varepsilon := 1-\pi(j)$.
Near the corner, every arm falls into one of three classes:
\begin{itemize}[leftmargin=*]
\item \textbf{Allies}\; $S^+ := \{i : r(i) > r(j)\}$:
  advantage $U(i) = \Delta_{ij} + O(\varepsilon) > 0$.
  These arms pull the mean reward \emph{up} past~$r(j)$.
\item \textbf{Harmful}\; $S^- := \{i : r(i) < r(j)\}$:
  advantage $U(i) = -\Delta_{ji} + O(\varepsilon) < 0$.
  These arms pull the mean reward \emph{down} below~$r(j)$,
  inflating the corner arm's advantage and reinforcing the trap.
\item \textbf{Ambiguous}\; $\{j\}$: the corner arm itself,
  with advantage $U(j) = O(\varepsilon)$.
\end{itemize}
PG's bad region exists precisely because harmful arms give the corner
arm a positive advantage.
Without harmful arms, the corner arm's advantage would be near zero,
and the optimal arm would dominate.

\paragraph{EG removes harmful arms.}
This is exactly what EG does: it gates out every arm with negative advantage.
Near the corner, the harmful arms in $S^-$ have $U(i) < 0$ and are
removed entirely.
Only allies (which help the optimal arm) and the ambiguous corner arm
(whose contribution is $O(\varepsilon^2)$) remain.
With the harmful arms gone:
\begin{enumerate}[leftmargin=*]
\item The optimal arm's logit receives a direct push proportional
      to $\pi(1)\,U(1) \ge \Omega(\varepsilon/K)$
      (from the sector condition $\pi(1) \ge \varepsilon/K$).
\item Every other ally $i \in S^+ \setminus \{1\}$ contributes
      \emph{non-negatively} to the optimal arm's logit gap, i.e., allies
      help escape. They never hinder it.
\item The corner arm's residual advantage is only $O(\varepsilon)$,
      and its contribution to the logit gap is $O(\varepsilon^2)$.
\end{enumerate}
The $\Omega(\varepsilon/K)$ push from arm~1 dominates the
$O(\varepsilon^2)$ residual, so the optimal arm's logit
\emph{unconditionally} grows faster than the corner arm's logit.
The bad region vanishes. Section~\ref{sec:corner_escape} formalizes this picture into
a logarithmic escape bound and a global convergence guarantee
with an asymptotic $O(1/t)$ rate.

%%%%%%%%%%%%%%%%%%%%%%%%%%%%%%%%%%%%%%%%%%%%%%%%%%%%%%%%%%%%%%%%%%%%%%%%%%%%%%%%
%%%%%%%%%%%%%%%%%%%%%%%%%%%%%%%%%%%%%%%%%%%%%%%%%%%%%%%%%%%%%%%%%%%%%%%%%%%%%%%%
\section{Corner Escape}
\label{sec:corner_escape}

We now formalize the intuition from Section~\ref{subsec:intuition}.
The logit-gap identity (Section~\ref{subsec:logit_gap}) decomposes
$\dot\theta(1)-\dot\theta(j)$ into the arm-specific and
population-level contributions identified above.
The sector bounds (Section~\ref{subsec:sector}) show that this
decomposition yields
$\dot\theta(1)-\dot\theta(j) \ge \Omega(\varepsilon/K)$
near any sub-optimal corner, under both EG and finite-temperature DG.
Integrating the resulting differential inequality
(Section~\ref{subsec:escape}) gives a first-exit escape time
logarithmic in $\pi_0(j)/\pi_0(1)$.
We then establish that EG converges globally to the optimal
policy (Section~\ref{subsec:convergence}): monotonic value improvement,
a measure-zero characterization of the bad region, and corner
instability combine to give global convergence at an asymptotic
$O(1/t)$ rate. Full proofs are deferred to Appendix~\ref{app:proofs}.

%%%%%%%%%%%%%%%%%%%%%%%%%%%%%%%%%%%%%%%%%%%%%%%%%%%%%%%%%%%%%%%%%%%%%%%%%%%%%%%%
\subsection{Logit-Gap Decomposition}
\label{subsec:logit_gap}
The central analytical tool is the following identity, which isolates
arm-specific terms from the population-level drift in any gated policy
gradient.
It is verified by expanding and subtracting.
\begin{lemma}[Logit-gap identity]
\label{lem:logit_gap}
For any gated policy gradient with per-arm weights $w_i \ge 0$,
define the weighted drift $S := \sum_{i=1}^K w_i\,\pi(i)\,U(i)$. Then
\begin{equation}
\label{eq:logit_gap}
\dot\theta(a) - \dot\theta(b)
=
\underbrace{\sum_{i=1}^K w_i\,\pi(i)\,U(i)\bigl(\Ind\{a{=}i\} - \Ind\{b{=}i\}\bigr)}_{\text{direct: arm-specific contributions}}
\;+\;
\underbrace{(\pi(b)-\pi(a))\,S}_{\text{indirect: population drift}}.
\end{equation}
\end{lemma}
Setting $a=1$ (optimal) and $b=j$ (corner), the identity reveals
two routes to escape.
The \emph{direct} term captures arm-specific contributions:
arm~1 pushes its own logit up, while arm~$j$ pushes its own logit up.
The \emph{indirect} term captures the population drift~$S$, amplified by
$\pi(j)-\pi(1) \approx 1$ near the corner.
For vanilla PG ($w_i = 1$), $S = 0$ because advantages are centered.
Under EG ($w_i = \Ind\{U(i)>0\}$), all harmful arms are removed while
at least one ally remains active, so $S > 0$ and the population drift
actively aids escape.

%%%%%%%%%%%%%%%%%%%%%%%%%%%%%%%%%%%%%%%%%%%%%%%%%%%%%%%%%%%%%%%%%%%%%%%%%%%%%%%%
\subsection{Sector Bounds}
\label{subsec:sector}
We now apply the logit-gap identity to prove that the optimal arm
dominates the corner arm on a quantitative sector near any sub-optimal
corner, first for EG, then for finite-temperature DG.

% \subsubsection{Enlightened Gradient}
The sector condition $\pi(1) \ge \varepsilon/K$ ensures arm~1 has enough
mass to contribute. The $1/K$ factor accounts for the worst-case
distribution of the remaining mass across $K-1$ arms.
\begin{theorem}[EG sector bound, general $K$]
\label{thm:corner_escape}
Under Assumption~\ref{assump:bandit}, fix any sub-optimal corner arm
$j \in \{2, \ldots, K\}$.
Write $\varepsilon = 1 - \pi(j)$ and let $p_a := \pi(a)$.
There exists $\varepsilon_0 > 0$ depending on the reward gaps and $K$
such that, for all $\varepsilon \in (0, \varepsilon_0)$,
if $p_1 \ge \varepsilon/K$, then under EG:
\[
\dot\theta(1) - \dot\theta(j) \ge \frac{\Delta_{1j}}{4K}\,\varepsilon.
\]
\end{theorem}
\paragraph{Proof sketch.}
Apply Lemma~\ref{lem:logit_gap} with $a=1$, $b=j$, and
$w_i = \Ind\{U(i)>0\}$.
As outlined in Section~\ref{subsec:intuition}, the logit gap decomposes
into three groups:
\[
\dot\theta(1)-\dot\theta(j)
=
\underbrace{p_1U(1)(1+p_j-p_1)}_{\ge\,\Delta_{1j}\varepsilon/(2K)}
\;+\;
\underbrace{\sum_{i\in S^+\setminus\{1\}} p_iU(i)(p_j-p_1)}_{\ge\,0\;\text{(allies)}}
\;-\;
\underbrace{\Ind\{U(j){>}0\}\,p_jU(j)(1+p_1-p_j)}_{O(\varepsilon^2)\;\text{(residual)}}.
\]
The arm-$1$ contribution $\Omega(\varepsilon/K)$ dominates
the $O(\varepsilon^2)$ residual, proving the sector gap.
\begin{remark}[All better-than-$j$ arms are allies]
\label{rem:allies}
Each arm $i \in S^+ \setminus \{1\}$ contributes
$p_i U(i)(p_j - p_1) \ge 0$, since $U(i) > 0$ and $p_j - p_1 \approx 1$.
The number of allies does not degrade the bound, and they enter only through
a non-negative sum that we drop.
The $1/K$ factor comes solely from the sector condition $p_1 \ge \varepsilon/K$.
\end{remark}

% \subsubsection{Finite Temperature}
EG removes harmful arms entirely.
At finite temperature $\eta > 0$, DG instead suppresses them polynomially:
the surprisal factor makes the gate weight on a rare arm vanish as that
arm becomes rarer.
\begin{lemma}[Polynomial suppression]
\label{lem:poly_suppress}
For any arm $k$ with $U(k) < 0$, the DG gate satisfies
$w(k) \le \pi(k)^{|U(k)|/\eta}$.
Near the corner, for $k \in S^-$ with $|U(k)| \ge \Delta_{jk}/2$:
\begin{equation}
\label{eq:poly_suppress}
w(k)\,\pi(k) \le \pi(k)^{1 + \Delta_{jk}/(2\eta)}.
\end{equation}
\end{lemma}
Because $\pi(k) \le \varepsilon$ for $k \in S^-$, the product
$w(k)\pi(k)$ is $o(\varepsilon)$: it vanishes faster than the
beneficial arm-$1$ term grows.
Summing over all harmful arms, the total contamination is negligible
compared to arm~1's contribution.
\begin{theorem}[DG sector bound, general $K$]
\label{thm:sigmoid_escape}
Under Assumption~\ref{assump:bandit}, fix any sub-optimal corner arm $j$
and any temperature $\eta > 0$.
There exists $\varepsilon_0(\eta) > 0$ such that,
for all $\varepsilon \in (0, \varepsilon_0(\eta))$,
if $p_1 \ge \varepsilon/K$, then DG satisfies
\[
\dot\theta(1) - \dot\theta(j) \ge \frac{\Delta_{1j}}{8K}\,\varepsilon.
\]
\end{theorem}
\paragraph{Proof sketch.}
The EG proof carries through unchanged for the beneficial terms.
The only new ingredient is bounding the harmful-arm contamination:
for each $k \in S^-$,
$w(k)\pi(k) \le \pi(k)^{1+\Delta_{jk}/(2\eta)}$,
which is higher order than $\varepsilon$.
The harmful terms become negligible for sufficiently small $\varepsilon$,
preserving the sector gap.
\begin{remark}[Surprisal is structurally necessary]
\label{rem:surprisal}
Without surprisal, $w_k = \sigma(-\Delta_{jk}/\eta)$ is a nonzero constant.
The $S^-$ contribution is $\Theta(\varepsilon)$, the same order as the
beneficial term, and the sector bound breaks.
The product $U \cdot \surp$ is required for the rarity-sensitive
suppression: without it, harmful arms do not become negligible
as they become rare.
\end{remark}

%%%%%%%%%%%%%%%%%%%%%%%%%%%%%%%%%%%%%%%%%%%%%%%%%%%%%%%%%%%%%%%%%%%%%%%%%%%%%%%%
\subsection{Escape Time}
\label{subsec:escape}
The sector bounds guarantee that the log-ratio $\log(\pi(1)/\pi(j))$
grows at rate $\Omega(\varepsilon)$ near any sub-optimal corner.
To convert this into an escape time, we need $\varepsilon$ to be
non-decreasing while the trajectory remains in the sector,
so the differential inequality can be integrated.
\begin{lemma}[Sector monotonicity]
\label{lem:sector_mono}
Under Assumption~\ref{assump:bandit} with EG, fix any sub-optimal
corner arm~$j$.
There exists $\bar\varepsilon > 0$ such that on the sector
$\Sc_{\bar\varepsilon} := \{\pi : 0 < \varepsilon < \bar\varepsilon,\;
\pi(1) \ge \varepsilon/K\}$:
\begin{enumerate}
\item $\frac{d}{dt}\log\frac{\pi(1)}{\pi(j)}
      \ge \frac{\Delta_{1j}}{4K}\,\varepsilon$,
\item $\dot\varepsilon > 0$.
\end{enumerate}
\end{lemma}
Because $\varepsilon$ is non-decreasing, we can lower-bound the
integrand by $\varepsilon_0 = 1 - \pi_0(j)$, yielding a bound
logarithmic in the initial probability ratio, an exponential
improvement over PG's corner transient.
\begin{theorem}[First-exit escape time]
\label{thm:escape_time}
Under Assumption~\ref{assump:bandit} with EG,
let $\bar\varepsilon$ be as in Lemma~\ref{lem:sector_mono} and
$\pi_0 \in \Sc_{\bar\varepsilon}$.
Write $\varepsilon_0 := 1-\pi_0(j)$.
Let $\Wc := \{\pi : \pi(1)/\pi(j) \ge 1\}$ be the parity set,
$\tau_{\Wc}$ the first time the trajectory enters $\Wc$,
and $\tau_{\mathrm{exit}}$ the first time it leaves $\Sc_{\bar\varepsilon}$.
Then
\[
\tau_{\Wc} \wedge \tau_{\mathrm{exit}}
\le
\frac{4K}{\Delta_{1j}\,\varepsilon_0}
\log\!\left(
\frac{\pi_0(j)}{\pi_0(1)}
\right).
\]
\end{theorem}
\paragraph{Proof sketch.}
While the trajectory remains in the sector,
Lemma~\ref{lem:sector_mono} gives
$\frac{d}{dt}\log(\pi(1)/\pi(j)) \ge \Delta_{1j}\varepsilon(t)/(4K)$,
and $\varepsilon(t) \ge \varepsilon_0$ (non-decreasing).
Integrating from $\log(\pi_0(1)/\pi_0(j)) < 0$ up to $0$ yields the bound.

\subsection{Gobal Convergence}
\label{subsec:convergence}

The escape time is logarithmic in $\pi_0(j)/\pi_0(1)$:
how deep the policy is trapped in the wrong corner.
This replaces PG's potentially exponential corner transient with a
logarithmic one.
Once the trajectory exits the corner sector and enters a regime where
arm~1 has the dominant advantage, the standard $O(1/t)$ asymptotics
take over~\citep{mei2020convergence}.
We now make this precise: Lemma~\ref{lem:corner_dominance} shows
that the bad region has measure zero, Lemma~\ref{lem:eg_bandit_discrete_monotonicity}
establishes monotonic progress, and Theorem~\ref{thm:global_conv}
combines them into a global convergence guarantee.

\begin{lemma}[Sub-optimal Corner Dominance]
\label{lem:corner_dominance}
Consider a $K$-armed bandit where $a^*=1$. In the neighborhood of any sub-optimal corner $j \neq 1$, the EG update ensures that the optimal arm's logit outpaces the corner arm's logit ($\Delta \theta(1) > \Delta \theta(j)$) for almost all points in the interior of the simplex. Specifically, the set of points where the optimal arm fails to dominate the sub-optimal favorite has \textbf{measure zero}, as it is confined to the higher-order boundary region where $\pi_1 \le O(\sum \pi_k^2)$.
\end{lemma}

\paragraph{Proof sketch.}
The logit-gap difference under EG decomposes as
\begin{equation}
    \Delta \theta(1) - \Delta \theta(j) = \Delta_{1j} C_{1j} + \sum_{k \neq 1, j} \Delta_{jk} C_{jk},
\end{equation}
where $C_{1j} = \pi_1 [ (1-\pi_1)^2 + \pi_j(2-\pi_j) ]$ and $C_{jk} = \pi_k (\pi_1 - \pi_j) \sum_{m \neq 1, j} \pi_m$. Near the corner ($\pi_j \to 1$), $C_{1j} \approx 2\pi_1$ while $C_{jk} = O(\pi_k^2)$, giving
\begin{equation}
    \Delta \theta(1) - \Delta \theta(j) \approx 2\pi_1 \Delta_{1j} - \sum_{k \neq 1, j} \pi_k^2 \Delta_{jk}.
\end{equation}
The bad region requires $\pi_1 \le O(\sum \pi_k^2)$, a higher-order boundary effect with zero Lebesgue measure.

To prove global convergence, we first establish that EG is a monotonic ascent direction using the covariance between updates and rewards.

\begin{lemma}[Discrete Monotonicity of EG in Bandits]
\label{lem:eg_bandit_discrete_monotonicity}
For a $K$-armed bandit, let $\pi'$ be the policy after a discrete EG update with step size $\eta > 0$. The change in expected reward is non-negative:
\begin{equation}
    V^\prime - V = \frac{\text{Cov}_{\pi}(e^{\eta f}, U)}{\mathbb{E}_{\pi}[e^{\eta f}]} = \frac{1}{Z} \sum_{i: U_i > 0} \pi_i U_i (e^{\eta U_i} - 1) \ge 0,
\end{equation}
where $U_i$ is the advantage of arm $i$, $f_i = U_i \cdot \mathbb{I}\{U_i > 0\}$ is the gated advantage, and $Z = \mathbb{E}_{\pi}[e^{\eta f}]$ is the partition function.
\end{lemma}

\begin{theorem}[Global Convergence]
\label{thm:global_conv}
For a $K$-armed bandit with distinct rewards and a sufficiently small constant learning rate $\eta$, EG converges to the optimal one-hot policy $\pi^*$.
\end{theorem}
\paragraph{Proof sketch.}
By Lemma~\ref{lem:eg_bandit_discrete_monotonicity}, $V_t$ is non-decreasing and bounded,
so $V_{t+1} - V_t \to 0$.
Since each term $\pi_t(i)\,U_i\,(e^{\eta U_i}-1)\,\Ind\{U_i > 0\}$ in the progress
formula is non-negative and their sum vanishes, every arm $i$ with $\pi_\infty(i) > 0$
must satisfy $U_i = 0$, i.e., $r(i) = V_\infty$.
Distinct rewards then force $\pi_t$ to converge to a one-hot policy.
By Lemma~\ref{lem:corner_dominance}, every sub-optimal corner $e_j$ is an unstable
repeller, i.e., $\Delta\theta(1) > \Delta\theta(j)$ almost everywhere nearby, so the limit
must be $\pi^* = e_1$.

Finally, noting $U_1 = V^* - V_t$ gives
$\delta_t - \delta_{t+1} \ge c\,\pi_t(1)\,\delta_t^2$;
once $\pi_t(1) \ge 1/2$ the standard reciprocal telescope yields $\delta_t = O(1/t)$.
\begin{corollary}[Asymptotic $O(1/t)$ convergence rate]
\label{cor:rate}
Under the conditions of Theorem~\ref{thm:global_conv}, the sub-optimality
$\delta_t := V^* - V_t$ satisfies $\delta_t = O(1/t)$.
\end{corollary}

\begin{remark}[Comparison with standard PG convergence proofs]
\label{rem:proof_comparison}
For standard softmax PG, the bad region near each sub-optimal corner
has \emph{positive} Lebesgue measure (Eq.~\ref{eq:pg_ratio_bound})~\citep{mei2020convergence},
so trajectories can initially be attracted toward sub-optimal
corners.
Proving $\pi_t \to \pi^*$ therefore requires an indirect argument:
one must show that the trajectory eventually escapes every bad region
despite spending time inside it, typically via a contradiction
involving the strict monotonicity of $V_t$ and the structure of the
gradient field on the boundary of the bad region.
For EG and DG, the bad region has \emph{measure zero}
(Lemmas~\ref{lem:corner_dominance}--\ref{lem:corner_dominance_dg}),
so sub-optimal corners are unconditional repellers for almost every
initialization.
The convergence proof reduces to three short steps:
$V_t$ converges (monotone bounded), the limit is one-hot
(distinct rewards), and every sub-optimal one-hot is unstable
(measure-zero bad region).
No contradiction-based escape argument is needed.
\end{remark}

\paragraph{Extension to finite temperature.}
Theorems~\ref{thm:global_conv}--\ref{thm:mdp_convergence} and
Corollaries~\ref{cor:rate}--\ref{cor:mdp_rate} are stated for EG
(the zero-temperature limit).
The same global convergence and $O(1/t)$ rate guarantees hold for DG
at any finite temperature $\eta \in (0,\infty)$:
the sign-preserving structure $w(a) > 0$ ensures monotonic value
improvement, and the polynomial suppression of harmful arms
(Lemma~\ref{lem:poly_suppress}) preserves the measure-zero bad region.
Full statements and proofs are given in Appendix~\ref{app:dg_convergence}.

%%%%%%%%%%%%%%%%%%%%%%%%%%%%%%%%%%%%%%%%%%%%%%%%%%%%%%%%%%%%%%%%%%%%%%%%%%%%%%%%
\section{Extension to Finite MDPs}
\label{sec:mdp_extension}

In finite MDPs, a sub-optimal ``corner'' corresponds to a deterministic sub-optimal policy $\Pi_j$ with $\pi(j(s)|s) = 1 - \epsilon_s$ for all $s$. The corner dominance property translates state-by-state.

\begin{lemma}[Local Corner Escape in MDPs]
\label{lem:mdp_local_escape}
Consider a tabular finite MDP with EG and fix a deterministic sub-optimal
policy $\Pi_j$, where $j(s)$ denotes the action chosen at state~$s$.
At any state $s$ where $a^*(s) := \arg\max_a Q^{\pi}(s,a) \neq j(s)$,
the EG update satisfies $\Delta\theta(s, a^*(s)) > \Delta\theta(s, j(s))$
for almost all local policies $\pi(\cdot|s)$ in the interior of the
per-state simplex $\Delta_K$.
Specifically, the set of points where the locally best action fails to
dominate the corner action has \textbf{measure zero} in $\Delta_K$,
as it is confined to the higher-order boundary region
\[
\pi(a^*(s)|s) \;\le\; O\Big(\sum_{k \neq a^*(s),\, j(s)} \pi(k|s)^2\Big).
\]
Consequently, the bad region in the full product simplex
$\prod_{s \in \mathcal{S}} \Delta_K$, where \emph{any} state fails the
dominance condition, is a finite union of measure-zero sets and hence
has measure zero.
\end{lemma}

EG preserves the update signal at these corners by focusing on positive-advantage actions. We now prove that local improvements aggregate into global value growth.

\begin{lemma}[Discrete Monotonicity of EG in MDPs]
\label{lem:mdp_monotonicity}
Consider a tabular finite MDP with the discrete EG update: at each state~$s$,
the policy is updated as
$\pi'(a|s) = \pi(a|s)\,e^{\eta f_a(s)}\big/Z_s$,
where $f_a(s) := [U_a(s)]_+ = U_a(s)\cdot\Ind\{U_a(s)>0\}$
and $Z_s := \sum_{a'}\pi(a'|s)\,e^{\eta f_{a'}(s)}$.
Then the global value function is monotonically non-decreasing:
\begin{equation}
V(\pi') - V(\pi)
= \frac{1}{1-\gamma}\sum_{s \in \mathcal{S}} d_{\rho}^{\pi'}(s)\;\frac{1}{Z_s}
  \sum_{i:\,U_i(s)>0} \pi(i|s)\,U_i(s)\,\bigl(e^{\eta U_i(s)}-1\bigr)
\;\ge\; 0.
\end{equation}
\end{lemma}

\begin{theorem}[MDP Global Convergence]
\label{thm:mdp_convergence}
Consider a tabular finite MDP with a unique optimal policy $\pi^*$,
where all states are reachable: $d_{\rho}^{\pi}(s) \ge d_{\min} > 0$
for all $s \in \mathcal{S}$ and all policies~$\pi$ with full action support.
Under the discrete EG update
$\pi'(a|s) = \pi(a|s)\,e^{\eta\,[U_a(s)]_+}/Z_s$
with a sufficiently small constant step size~$\eta$:
\begin{enumerate}
\item $V(\pi_t) \to V^* = V(\pi^*)$,
\item $\pi_t \to \pi^*$.
\end{enumerate}
\end{theorem}

\paragraph{Proof sketch.}
By Lemma~\ref{lem:mdp_monotonicity}, $V(\pi_t)$ is non-decreasing and bounded,
so $V_{t+1} - V_t \to 0$.
Since each per-state term
$\pi_t(a|s)\,[U_a(s)]_+\,(e^{\eta[U_a]_+}-1)$
in the progress formula is non-negative and their weighted sum vanishes,
every action $a$ with $\bar\pi(a|s)>0$ at any limit point $\bar\pi$
must satisfy $Q^{\bar\pi}(s,a) \le V^{\bar\pi}(s)$.
If this held with equality at every reachable state,
$V^{\bar\pi}$ would satisfy the Bellman optimality equation,
forcing $V^{\bar\pi} = V^*$ by the contraction property, so
any sub-optimal limit point must place zero mass on some improving action.
But the EG logit update $\theta'(s,a) = \theta(s,a) + \eta\,[U_a(s)]_+$
then grows the improving action's logit at rate $\Omega(\delta)$
while the corner action's logit stalls,
driving $\pi_t(\tilde a|s)/\pi_t(j|s) \to \infty$
and contradicting convergence to the sub-optimal corner.
Therefore every limit point is optimal, and by uniqueness $\pi_t \to \pi^*$.

\begin{corollary}[Asymptotic $O(1/t)$ convergence rate in MDPs]
\label{cor:mdp_rate}
Under the conditions of Theorem~\ref{thm:mdp_convergence},
the sub-optimality $\delta_t := V^* - V(\pi_t)$ satisfies $\delta_t = O(1/t)$.
\end{corollary}

This global convergence is robust even in ``hard'' MDPs where reward gaps are small.

%%%%%%%%%%%%%%%%%%%%%%%%%%%%%%%%%%%%%%%%%%%%%%%%%%%%%%%%%%%%%%%%%%%%%%%%%%%%%%%%
%%%%%%%%%%%%%%%%%%%%%%%%%%%%%%%%%%%%%%%%%%%%%%%%%%%%%%%%%%%%%%%%%%%%%%%%%%%%%%%%
\section{Experiments}
\label{sec:experiments}

The theory predicts that DG escape time scales gently with problem hardness, while PG's grows superlinearly.
We test this first in a tabular bandit where the gradient flow is exact, then ask whether the mechanism survives shared neural-network parameters.

\paragraph{Tabular escape-time scaling.}
We simulate the continuous-time gradient flow for a $3$-armed bandit with rewards $(1,\; 1{-}\Delta,\; 0)$ from a fixed bad initialization $\theta_0 = (-1, 5, 1)$, so that $\pi_0(2) \approx 0.94$.
Sweeping $\Delta$ from $0.5$ to $0.01$, we measure the time for the optimal arm's logit to overtake the corner arm's.
PG escape time grows superlinearly in $1/\Delta$ (Figure~\ref{fig:escape_tabular}), consistent with the exponential lower bounds of~\citet{li2023softmax}.
DG remains near-linear: the delight gate removes the pathological dependence on problem hardness.

\paragraph{Neural corner recovery.}
We use the MNIST contextual bandit~\citep{osband2026delightfulpolicygradient}: each image is a context, the $10$ digit labels are arms, and reward is $1$ for correct classification.
The policy is an MLP with two hidden layers of width $100$ and shared parameters across all contexts.
To create a controlled corner, we pretrain with cross-entropy toward a random wrong digit for $\{1, 10, 10^2, 10^3, 10^4\}$ steps, then train for $T{=}10\text{k}$ steps with PG or DG.
DG classification error is flat across corner depth; PG degrades steadily (Figure~\ref{fig:corner_mnist}).
The mechanism transfers to shared-parameter networks, even though the tabular guarantees do not formally apply (Appendix~\ref{sec:counterexample}).
Full experimental details are in Appendix~\ref{app:experiments}.

\begin{figure}[ht!]
\centering
\begin{minipage}[t]{0.48\linewidth}
    \centering
    \includegraphics[width=\linewidth]{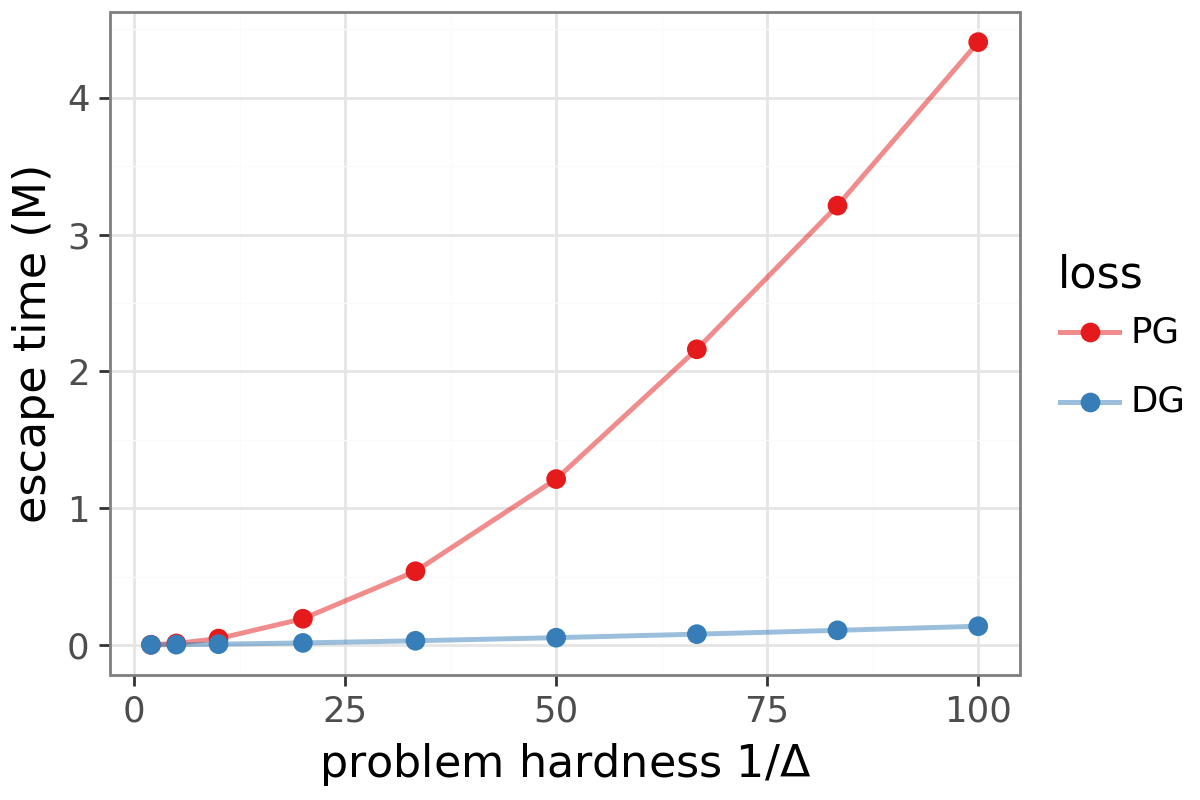}
    \captionof{figure}{Escape time vs.\ $1/\Delta$: PG grows superlinearly; DG stays near-linear.}
    \label{fig:escape_tabular}
\end{minipage}\hfill
\begin{minipage}[t]{0.48\linewidth}
    \centering
    \includegraphics[width=\linewidth]{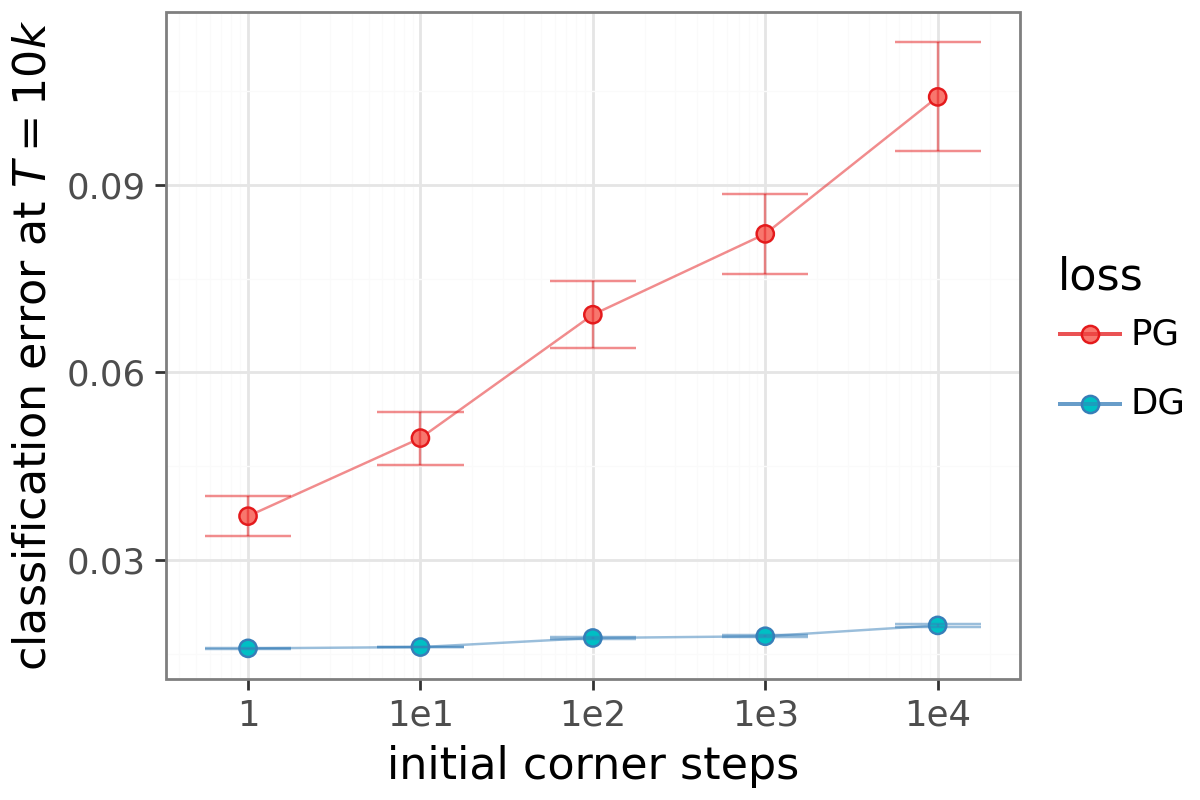}
    \captionof{figure}{MNIST error at $T{=}10\text{k}$ vs.\ corner depth. Bars: $\pm 1$ s.e.\ ($100$ runs).}
    \label{fig:corner_mnist}
\end{minipage}
\end{figure}

%%%%%%%%%%%%%%%%%%%%%%%%%%%%%%%%%%%%%%%%%%%%%%%%%%%%%%%%%%%%%%%%%%%%%%%%%%%%%%%%
%%%%%%%%%%%%%%%%%%%%%%%%%%%%%%%%%%%%%%%%%%%%%%%%%%%%%%%%%%%%%%%%%%%%%%%%%%%%%%%%
\section{Conclusion}
\label{sec:conclusion}

Softmax policy gradient self-traps near sub-optimal corners: the corner arm's logit grows faster than the optimal arm's because $\pi(j) \approx 1$ amplifies the corner arm's advantage while $\pi(1) \approx 0$ throttles the optimal arm's.
The enlightened gradient (the zero-temperature limit of DG) removes this trap.
We prove a sector bound showing the optimal arm's logit grows faster than the corner arm's near any sub-optimal corner, yielding a first-exit escape time logarithmic in the initial probability ratio.
At any finite temperature, DG retains the same local sector gap: the surprisal factor polynomially suppresses harmful arms, making their contribution negligible.
A structural insight underlies the proof: all arms better than the corner arm are allies whose contributions are non-negative, so the number of arms degrades constants only as $O(1/K)$.
Beyond the sector analysis, we establish global convergence of EG for both bandits and tabular MDPs, with an asymptotic $O(1/t)$ rate once the corner is escaped.

Under function approximation with shared parameters, the tabular mechanism need not survive: we isolate three conditions and provide an exact counterexample (Appendix~\ref{sec:counterexample}).
This should be read as a limitation of the theorem rather than a blanket negative claim about DG with shared representations.
Whether DG can overcome such coupling under structured function approximation, for instance, when the conflicting states are rarely visited, remains an open question.
The product of advantage and surprisal is necessary for this rarity-sensitive suppression; without it, the gate assigns harmful arms a constant weight and the mechanism fails.

%%%%%%%%%%%%%%%%%%%%%%%%%%%%%%%%%%%%%%%%%%%%%%%%%%%%%%%%%%%%%%%%%%%%%%%%%%%%%%%%
% BIBLIOGRAPHY
%%%%%%%%%%%%%%%%%%%%%%%%%%%%%%%%%%%%%%%%%%%%%%%%%%%%%%%%%%%%%%%%%%%%%%%%%%%%%%%%
{\small
\bibliographystyle{plainnat}
\bibliography{references}
}
\newpage

%%%%%%%%%%%%%%%%%%%%%%%%%%%%%%%%%%%%%%%%%%%%%%%%%%%%%%%%%%%%

\appendix

%%%%%%%%%%%%%%%%%%%%%%%%%%%%%%%%%%%%%%%%%%%%%%%%%%%%%%%%%%%%%%%%%%%%%%%%%%%%%%%%
%%%%%%%%%%%%%%%%%%%%%%%%%%%%%%%%%%%%%%%%%%%%%%%%%%%%%%%%%%%%%%%%%%%%%%%%%%%%%%%%
\section{Function Approximation Counterexample}
\label{sec:counterexample}

The corner escape theorems hold for tabular bandits, where each arm's logit can be adjusted independently.
Under function approximation with shared parameters, a single parameter update can improve the policy in one state while degrading it in another.
When the advantage of the same action has opposite signs in different states, the gate that protects the corner in one state can suppress the correction needed in the other.
We construct a minimal example where this coupling creates a sub-optimal interior fixed point for DG, while PG still converges correctly.

The example uses a two-state MDP with a single shared parameter controlling the probability of action $a_1$ in both states.
In state $s_1$, action $a_1$ is mildly beneficial; in state $s_2$, it is catastrophically harmful.

\begin{example}[Shared-parameter MDP]
\label{ex:counterexample}
State $s_0$ transitions to $s_1$ or $s_2$ with probability $1/2$ each, then terminates.
Rewards: $r(s_1, a_1) = +10$, $r(s_1, a_0) = 0$, $r(s_2, a_1) = -100$, $r(s_2, a_0) = 0$.
Shared parameter: $\pi_\theta(a_1|s) = \sigma(\theta)$ for both states.
\end{example}

This counterexample should be read as a limitation result, not as a claim that DG is ineffective with shared function approximation in general.
Its role is to show that once parameters are shared across states, the tabular corner-escape mechanism is no longer guaranteed.
This is analogous to classical counterexamples in temporal-difference learning with function approximation, which identify genuine failure modes without implying that function approximation should be avoided in practice~\citep{tsitsiklis1997analysis,baird1995residual}.

The best-in-class policy sets $p = \sigma(\theta)$ to minimize overall loss.
For this example, the shared-parameter drifts are explicit:
\[
F_{\mathrm{PG}}(\theta) = -45p(1-p),
\qquad
F_{\mathrm{EG}}(\theta) = 5p(1-p)(1-11p).
\]
PG correctly drives $p \to 0$ because $F_{\mathrm{PG}}(\theta) < 0$ for all $p \in (0,1)$, driven by the large negative reward in $s_2$.
Under EG, $F_{\mathrm{EG}}$ captures only the positive-advantage terms, yielding a stable root at $p^* = 1/11$.
Similarly under DG, the gate suppresses the $s_2$ correction, it has negative advantage and high surprisal, allowing the small positive advantage from $s_1$ to create an interior fixed point.

\begin{theorem}[A shared-parameter counterexample for DG]
\label{thm:counterexample}
In Example~\ref{ex:counterexample}, write $p := \sigma(\theta)$. \textbf{1) PG converges to best-in-class.} \textbf{2) EG has a stable interior fixed point at $p^* = 1/11$.} \textbf{3) DG ($\eta = 1$) admits an interior fixed point empirically.}
Numerically, the DG fixed point occurs near $p^* \approx 0.116$.
\end{theorem}

The failure requires three conditions acting together.
First, shared parameters couple the policies across states, so a single update affects both.
Second, the same action has conflicting advantages: positive in $s_1$, negative in $s_2$.
Third, the gate suppresses the larger-magnitude correction because it has negative delight.
In tabular settings, the first condition never holds: each state's logits are independent, and the corner escape theorem applies state by state.

The point of the example is therefore not that shared parameters make DG unusable.
Rather, it shows that the tabular mechanism proved in Sections~\ref{sec:algorithms}--\ref{sec:corner_escape} does not extend automatically once states are coupled through a shared representation.
In many neural-network settings DG may still outperform PG in practice; the counterexample only shows that such behavior cannot be taken as a theorem of the present form.

%%%%%%%%%%%%%%%%%%%%%%%%%%%%%%%%%%%%%%%%%%%%%%%%%%%%%%%%%%%%

%%%%%%%%%%%%%%%%%%%%%%%%%%%%%%%%%%%%%%%%%%%%%%%%%%%%%%%%%%%%%%%%%%%%%%%%%%%%%%%%
%%%%%%%%%%%%%%%%%%%%%%%%%%%%%%%%%%%%%%%%%%%%%%%%%%%%%%%%%%%%%%%%%%%%%%%%%%%%%%%%
\section{Experimental Details}
\label{app:experiments}

We provide implementation details and reproducibility information for the experiments in Section~\ref{sec:experiments}.
All experiments use JAX~\citep{jax2018github} and run on a single CPU; each completes in minutes.

\paragraph{Tabular escape-time scaling.}

The tabular experiment integrates the continuous-time gradient flow ODE for a $K{=}3$ armed bandit using forward Euler with step size $\mathrm{dt} = 1.0$.
The rewards are $r = (1,\; 1{-}\Delta,\; 0)$ and the initial logits are $\theta_0 = (-1, 5, 1)$, giving $\pi_0 \approx (0.01, 0.94, 0.05)$.
We sweep $\Delta \in \{0.5, 0.2, 0.1, 0.05, 0.03, 0.02, 0.015, 0.012, 0.01\}$ and integrate for up to $10^7$ time units.
Escape is declared when $\theta(1) \ge \theta(2)$, i.e., the optimal arm's logit overtakes the corner arm's.

The logit-space drift for each method is:
\begin{align*}
\dot{\theta}(a) &= \pi(a)\bigl[g(a) \cdot U(a) - \textstyle\sum_b \pi(b)\, g(b)\, U(b)\bigr],
\end{align*}
where $U(a) = r(a) - \pi^\top r$ is the advantage and $g(a)$ is the gate:
\begin{itemize}[nosep,leftmargin=*]
\item \textbf{PG}: $g(a) = 1$ for all $a$.
\item \textbf{EG}: $g(a) = \mathbf{1}\{U(a) > 0\}$.
\item \textbf{DG} ($\eta$): $g(a) = \sigma\!\bigl(U(a)\cdot\ell(a)\,/\,\eta\bigr)$, where $\ell(a) = -\log\pi(a)$.
\end{itemize}
DG uses $\eta = 1$ throughout.
The complete experiment code is shown below.

\begin{lstlisting}[language=Python, basicstyle=\ttfamily\scriptsize, breaklines=true, frame=single, caption={Complete tabular escape-time experiment (JAX).}, label={lst:tabular}]
import functools
import jax
import jax.numpy as jnp
import numpy as np
import pandas as pd

def bandit_drift(theta, rewards, method='pg', eta=1.0):
  pi = jax.nn.softmax(theta)
  adv = rewards - jnp.dot(pi, rewards)
  if method == 'pg':
    gate = jnp.ones_like(pi)
  elif method == 'eg':
    gate = (adv > 0).astype(jnp.float32)
  elif method == 'dg':
    surp = -jnp.log(jnp.maximum(pi, 1e-30))
    gate = jax.nn.sigmoid(adv * surp / eta)
  signal = gate * adv
  return pi * signal - pi * jnp.sum(signal * pi)

@functools.partial(jax.jit, static_argnames=['method', 'max_steps'])
def simulate_escape(theta0, rewards, method, eta, dt, max_steps):
  def step_fn(carry, _):
    theta, escaped, esc_step, idx = carry
    drift = bandit_drift(theta, rewards, method=method, eta=eta)
    new_theta = jnp.where(escaped, theta, theta + dt * drift)
    now_esc = new_theta[0] >= new_theta[1]
    new_step = jnp.where(escaped | ~now_esc, esc_step, idx)
    return (new_theta, escaped | now_esc, new_step, idx + 1), None
  init = (theta0, jnp.bool_(False), jnp.int32(max_steps), jnp.int32(0))
  (_, escaped, esc_step, _), _ = jax.lax.scan(step_fn, init, None, length=max_steps)
  return esc_step.astype(jnp.float32) * dt, escaped

def run_gap_sweep(init_logits=np.array([-1., 5., 1.]),
                  gaps=np.array([.5,.2,.1,.05,.03,.02,.015,.012,.01]),
                  dt=1.0, max_time=10_000_000, eta=1.0):
  max_steps = int(max_time / dt)
  theta0 = jnp.tile(jnp.array(init_logits), (len(gaps), 1))
  rewards = jnp.array([[1., 1. - g, 0.] for g in gaps])
  sim = jax.vmap(simulate_escape, in_axes=(0,0,None,None,None,None))
  records = []
  for method in ['pg', 'dg']:
    times, escaped = sim(theta0, rewards, method, eta, dt, max_steps)
    for i, g in enumerate(gaps):
      records.append(dict(method=method, gap=float(g), inv_gap=1./g,
        escape_time=float(times[i]) if escaped[i] else np.nan,
        escaped=bool(escaped[i])))
  return pd.DataFrame(records)
\end{lstlisting}

\paragraph{Neural corner recovery.}

The neural experiment uses the MNIST contextual bandit framework.
Each of the $60{,}000$ training images is a decision context, the $10$ digit classes are the arms ($K{=}10$), and reward is $1$ for a correct prediction and $0$ otherwise.
The policy network is a two-hidden-layer MLP with widths $[100, 100]$, ReLU activations, and a softmax output layer.
All parameters are shared across all contexts.
We use Adam~\citep{kingma2014adam} with learning rate $10^{-3}$ and batch size $100$.

\paragraph{Corner pretraining.}
To create a controlled bad initialization, we pretrain the network using cross-entropy loss toward a fixed \emph{wrong} target digit $d_{\text{corner}} \in \{0, \ldots, 9\}$ for a specified number of corner steps $\in \{1, 10, 10^2, 10^3, 10^4\}$.
During corner pretraining, every image is assigned the label $d_{\text{corner}}$ regardless of its true class, and the network is trained to minimize the cross-entropy against these corrupted labels.
This drives the softmax policy deep into a corner where the network confidently predicts the wrong digit.

\paragraph{Recovery training.}
After corner pretraining, we train for $T = 10{,}000$ steps using either PG (REINFORCE) or DG ($\eta = 1$) with the expected-confidence baseline.
Both methods receive identical reward signals (1 for correct, 0 for incorrect).
We evaluate classification error on the full $10{,}000$-image test set every $100$ steps, and report the final error at step $T$.

\paragraph{Averaging.}
Each configuration (loss type $\times$ corner steps) is run for $10$ random seeds $\times$ $10$ corner digits, giving $100$ independent runs.
We report the mean and $\pm 1$ standard error of classification error at step $T$.

%%%%%%%%%%%%%%%%%%%%%%%%%%%%%%%%%%%%%%%%%%%%%%%%%%%%%%%%%%%%%%%%%%%%%%%%%%%%%%%%
\section{Proofs}
\label{app:proofs}

This appendix provides complete proofs for the theoretical results in the main text.

%%%%%%%%%%%%%%%%%%%%%%%%%%%%%%%%%%%%%%%%%%%%%%%%%%%%%%%%%%%%%%%%%%%%%%%%%%%%%%%%
\subsection{Proof of Theorem~\ref{thm:corner_escape} (EG sector bound)}

The proof applies the logit-gap identity and bounds three term groups: arm~1's direct contribution, the allies, and the arm-$j$ residual.

\begin{proof}
After shrinking $\varepsilon_0$ if needed, we may assume
$U(i) \ge \Delta_{ij}/2$ for all $i \in S^+$,
$|U(j)| \le C_U\varepsilon$ for some constant $C_U > 0$,
and $U(i) < 0$ for all $i \in S^-$.
Under EG, all $S^-$ arms are gated out.

Apply the logit-gap identity (Lemma~\ref{lem:logit_gap}) with $a = 1$, $b = j$, and $w_i = \Ind\{U(i) > 0\}$.
Since all $S^-$ arms are gated out, the sum runs over $i \in S^+ \cup (\{j\} \cap \{U(j) > 0\})$.
Writing $S = \sum_{i:\,U(i) > 0} p_i\,U(i)$, the identity gives:
\begin{align*}
\dot\theta(1) - \dot\theta(j)
&= p_1 U(1) - \Ind\{U(j) > 0\}\,p_j U(j) + (p_j - p_1)\,S.
\end{align*}
Since $S = p_1 U(1) + \Ind\{U(j) > 0\}\,p_j U(j) + \sum_{i \in S^+\setminus\{1\}} p_i U(i)$,
we rearrange:
\begin{equation}
\label{eq:three_term}
\dot\theta(1) - \dot\theta(j)
= \underbrace{p_1\,U(1)\,(1 + p_j - p_1)}_{\text{arm-$1$ contribution}}
  \;+\; \underbrace{\sum_{i \in S^+\setminus\{1\}} p_i\,U(i)\,(p_j - p_1)}_{\text{ally contributions $\ge 0$}}
  \;-\; \underbrace{\Ind\{U(j)>0\}\,p_j\,U(j)\,(1 + p_1 - p_j)}_{\text{arm-$j$ residual}}.
\end{equation}

\textbf{Term 1: arm-1 contribution.}
Since $p_1 \ge \varepsilon/K$, $U(1) \ge \Delta_{1j}/2$, and $1 + p_j - p_1 = 2 - \varepsilon - p_1 \ge 1$ (for $\varepsilon \le 1/2$):
\[
p_1\,U(1)\,(1 + p_j - p_1) \ge \frac{\varepsilon}{K} \cdot \frac{\Delta_{1j}}{2} \cdot 1 = \frac{\Delta_{1j}}{2K}\,\varepsilon.
\]

\textbf{Term 2: ally contributions.}
For each $i \in S^+ \setminus \{1\}$: $p_i > 0$, $U(i) \ge \Delta_{ij}/2 > 0$, and $p_j - p_1 = (1 - \varepsilon) - p_1 \ge 0$.
Therefore each ally term $p_i\,U(i)\,(p_j - p_1) \ge 0$.
We drop these non-negative terms for the lower bound.

\textbf{Term 3: arm-$j$ residual.}
When $U(j) > 0$: $p_j U(j) \le C_U\varepsilon$ and $1 + p_1 - p_j = \varepsilon + p_1 \le 2\varepsilon$, so
\[
\bigl|\Ind\{U(j) > 0\}\,p_j\,U(j)\,(1 + p_1 - p_j)\bigr| \le 2C_U\varepsilon^2.
\]

\textbf{Combine.}
\[
\dot\theta(1) - \dot\theta(j)
\ge \frac{\Delta_{1j}}{2K}\,\varepsilon - 2C_U\varepsilon^2.
\]
Choosing $\varepsilon_0 \le \Delta_{1j}/(8KC_U)$ yields
\[
\dot\theta(1) - \dot\theta(j)
\ge \frac{\Delta_{1j}}{4K}\,\varepsilon. \qedhere
\]
\end{proof}

%%%%%%%%%%%%%%%%%%%%%%%%%%%%%%%%%%%%%%%%%%%%%%%%%%%%%%%%%%%%%%%%%%%%%%%%%%%%%%%%
\subsection{Proof of Lemma~\ref{lem:poly_suppress} (Polynomial suppression)}

The bound follows from the monotonicity of the sigmoid and the structure of the delight product.

\begin{proof}
Since $U(k) < 0$ and $\surp(k) = -\log\pi(k) > 0$:
\[
w(k) = \sigma\!\bigl(U(k)\surp(k)/\eta\bigr) \le \exp\!\bigl(U(k)\surp(k)/\eta\bigr) = \exp\!\bigl(-|U(k)|\log(1/\pi(k))/\eta\bigr) = \pi(k)^{|U(k)|/\eta}.
\]
Near the corner, $|U(k)| \ge \Delta_{jk}/2$ for $k \in S^-$, so $w(k)\,\pi(k) \le \pi(k)^{1 + \Delta_{jk}/(2\eta)}$.
\end{proof}

%%%%%%%%%%%%%%%%%%%%%%%%%%%%%%%%%%%%%%%%%%%%%%%%%%%%%%%%%%%%%%%%%%%%%%%%%%%%%%%%
\subsection{Proof of Theorem~\ref{thm:sigmoid_escape} (DG sector bound)}

The proof mirrors Theorem~\ref{thm:corner_escape}, but we write the logit gap as the sum of four groups: the arm-$1$ term, ally terms from $S^+\setminus\{1\}$, the arm-$j$ residual, and the suppressed harmful terms from $S^-$.

\begin{proof}
We use the same corner parameterization as Theorem~\ref{thm:corner_escape}.
Applying Lemma~\ref{lem:logit_gap} with $a=1$ and $b=j$, we decompose the logit gap into the arm-$1$ term, the ally terms from $S^+\setminus\{1\}$, the arm-$j$ residual, and the harmful terms from $S^-$.

\textbf{$S^+$ arms (beneficial).}
For each $i \in S^+$: $U(i) \ge \Delta_{ij}/2 > 0$ and $\surp(i) = -\log p_i \to \infty$ as $\varepsilon \to 0$ (since $p_i \le \varepsilon$).
Therefore $U(i)\surp(i)/\eta \to +\infty$ and $w_i \to 1$.
After shrinking the neighborhood: $w_i \ge 1/2$ for all $i \in S^+$.

Arm~1 contributes $w_1 p_1 U(1)(1 + p_j - p_1)$ to the logit gap.
With $w_1 \ge 1/2$, $p_1 \ge \varepsilon/K$, $U(1) \ge \Delta_{1j}/2$, and $1 + p_j - p_1 \ge 1$:
\[
w_1 p_1 U(1)(1 + p_j - p_1)
\ge \frac{1}{2}\cdot\frac{\varepsilon}{K}\cdot\frac{\Delta_{1j}}{2}\cdot 1
= \frac{\Delta_{1j}}{4K}\,\varepsilon.
\]
Each $i \in S^+ \setminus\{1\}$ contributes $w_i p_i U(i)(p_j - p_1) \ge 0$ (same ally structure, since $w_i \ge 0$).

\textbf{Arm $j$ (residual).}
Same bound as in Theorem~\ref{thm:corner_escape}: $|T_j| \le 2C_U\varepsilon^2$.

\textbf{$S^-$ arms (suppressed).}
By Lemma~\ref{lem:poly_suppress}, each $k \in S^-$ satisfies $w_k p_k \le p_k^{1+\Delta_{jk}/(2\eta)} \le \varepsilon^{1+\Delta_{jk}/(2\eta)}$.
Since $|U(k)| \le R := r(1) - r(K)$ and $|p_j - p_1| \le 1$:
\[
\bigl|w_k p_k U(k)(p_j - p_1)\bigr|
\le R\,\varepsilon^{1+\Delta_{jk}/(2\eta)}.
\]
Summing over $S^-$ and writing $\delta_\eta := \min_{k \in S^-} \Delta_{jk}/(2\eta) > 0$:
\[
\sum_{k \in S^-}\bigl|w_k p_k U(k)(p_j - p_1)\bigr|
\le R|S^-|\,\varepsilon^{1+\delta_\eta} = o(\varepsilon).
\]

\textbf{Combine.}
\[
\dot\theta(1) - \dot\theta(j)
\ge \frac{\Delta_{1j}}{4K}\,\varepsilon - 2C_U\varepsilon^2 - R|S^-|\varepsilon^{1+\delta_\eta}.
\]
Both error terms are $o(\varepsilon)$.
Choosing $\varepsilon_0(\eta)$ so that $2C_U\varepsilon + R|S^-|\varepsilon^{\delta_\eta} \le \Delta_{1j}/(8K)$:
\[
\dot\theta(1) - \dot\theta(j) \ge \frac{\Delta_{1j}}{8K}\,\varepsilon. \qedhere
\]
\end{proof}

%%%%%%%%%%%%%%%%%%%%%%%%%%%%%%%%%%%%%%%%%%%%%%%%%%%%%%%%%%%%%%%%%%%%%%%%%%%%%%%%
\subsection{Proof of Lemma~\ref{lem:sector_mono} (Sector monotonicity)}

The lemma combines the sector bound with a direct calculation showing that all non-corner arms gain on the corner arm.

\begin{proof}
\textbf{Part (a): ratio growth.}
Immediate from Theorem~\ref{thm:corner_escape}:
$\frac{d}{dt}\log(\pi(1)/\pi(j)) = \dot\theta(1) - \dot\theta(j) \ge \Delta_{1j}\varepsilon/(4K)$.

\textbf{Part (b): $\dot\varepsilon > 0$.}
Since $\varepsilon = 1 - p_j$,
\[
\dot\varepsilon = -\dot p_j = p_j \sum_{a \neq j} p_a \bigl(\dot\theta(a) - \dot\theta(j)\bigr).
\]
It therefore suffices to show $\dot\theta(a) - \dot\theta(j) > 0$ for every $a \neq j$.
Under EG, for any $a \neq j$, the logit-gap identity gives
\[
\dot\theta(a) - \dot\theta(j)
= \Ind\{a \in S^+\}\,p_a U(a)
- \Ind\{U(j) > 0\}\,p_j U(j)
+ (p_j - p_a)\,S,
\]
where $S := \sum_{i:\,U(i)>0} p_i U(i) \ge p_1 U(1) \ge \frac{\varepsilon}{K} \cdot \frac{\Delta_{1j}}{2}$,
$p_j - p_a \ge 1 - 2\varepsilon$, and $p_j U(j) = O(\varepsilon)$.

If $a \in S^-$, then
\[
\dot\theta(a) - \dot\theta(j)
= (p_j - p_a)S - \Ind\{U(j) > 0\}\,p_j U(j)\,(1 - p_j + p_a)
\ge c\varepsilon - O(\varepsilon^2) > 0.
\]

If $a = 1$, the claim is exactly Theorem~\ref{thm:corner_escape}.

If $a \in S^+ \setminus \{1\}$, then
\[
\dot\theta(a) - \dot\theta(j)
= p_a U(a)(1 + p_j - p_a)
+ (p_j - p_a)\!\!\sum_{i \in S^+ \setminus \{a\}}\!\! p_i U(i)
- \Ind\{U(j) > 0\}\,p_j U(j)\,(1 - p_j + p_a),
\]
and the second term contains the arm-$1$ contribution, so again
$\dot\theta(a) - \dot\theta(j) \ge c\varepsilon - O(\varepsilon^2) > 0$.

Thus $\dot\varepsilon > 0$ for all sufficiently small $\varepsilon$.
\end{proof}

%%%%%%%%%%%%%%%%%%%%%%%%%%%%%%%%%%%%%%%%%%%%%%%%%%%%%%%%%%%%%%%%%%%%%%%%%%%%%%%%
\subsection{Proof of Theorem~\ref{thm:escape_time} (First-exit escape time)}

The escape time follows by integrating the differential inequality from Lemma~\ref{lem:sector_mono}.

\begin{proof}
Set $\rho(t) := \log(\pi_t(1)/\pi_t(j))$.
For $t < \tau_{\Wc} \wedge \tau_{\mathrm{exit}}$, the trajectory remains in $\Sc_{\bar\varepsilon}$ by definition of $\tau_{\mathrm{exit}}$.
By part~(a) and $\varepsilon(t) \ge \varepsilon_0$ (from part~(b)):
\[
\dot\rho(t) \ge \frac{\Delta_{1j}}{4K}\,\varepsilon(t) \ge \frac{\Delta_{1j}}{4K}\,\varepsilon_0.
\]
Integrating from $0$ to $\tau := \tau_{\Wc} \wedge \tau_{\mathrm{exit}}$:
\[
\rho(\tau) - \rho(0) \ge \frac{\Delta_{1j}}{4K}\,\varepsilon_0 \cdot \tau.
\]
The target threshold is $\rho = 0$ (i.e., $\pi(1) = \pi(j)$), giving
\[
\tau
\le
\frac{4K}{\Delta_{1j}\,\varepsilon_0}
\bigl(0 - \rho(0)\bigr)
=
\frac{4K}{\Delta_{1j}\,\varepsilon_0}
\log\!\left(
\frac{\pi_0(j)}{\pi_0(1)}
\right). \qedhere
\]
\end{proof}

%%%%%%%%%%%%%%%%%%%%%%%%%%%%%%%%%%%%%%%%%%%%%%%%%%%%%%%%%%%%%%%%%%%%%%%%%%%%%%%%
\subsection{Proof of Theorem~\ref{thm:counterexample} (A shared-parameter counterexample)}

The proof computes the gradient flow explicitly for each method.

\begin{proof}
Write $p := \sigma(\theta)$ for the shared action probability.
The per-state values are $V(s_1) = 10p$ and $V(s_2) = -100p$.
The advantages are:
\begin{align*}
U(s_1,a_1) &= 10(1{-}p), & U(s_1,a_0) &= -10p, \\
U(s_2,a_1) &= -100(1{-}p), & U(s_2,a_0) &= 100p.
\end{align*}
The score derivatives are $\partial_\theta \log\pi(a_1) = 1-p$ and $\partial_\theta\log\pi(a_0) = -p$.
The total gradient is $F(\theta) = \frac{1}{2}(G_{s_1} + G_{s_2})$.

\textbf{Part 1: PG.}
$G_s = \sum_a \pi(a) U(s,a) \partial_\theta\log\pi(a)$:
\begin{align*}
G_{s_1} &= p \cdot 10(1{-}p) \cdot (1{-}p) + (1{-}p) \cdot (-10p) \cdot (-p) = 10p(1{-}p), \\
G_{s_2} &= p \cdot (-100)(1{-}p) \cdot (1{-}p) + (1{-}p) \cdot 100p \cdot (-p) = -100p(1{-}p).
\end{align*}
$F_{\mathrm{PG}} = \frac{1}{2}(10 - 100)p(1{-}p) = -45p(1{-}p) < 0$ for all $p \in (0,1)$.
Therefore $\theta \to -\infty$ and $p \to 0$.

\textbf{Part 2: EG.}
The gate keeps $a_1$ at $s_1$ (positive advantage) and $a_0$ at $s_2$ (positive advantage):
\begin{align*}
G_{s_1}^{\mathrm{EG}} &= p \cdot 10(1{-}p)^2, \\
G_{s_2}^{\mathrm{EG}} &= (1{-}p)(100p)(-p) = -100p^2(1{-}p).
\end{align*}
$F_{\mathrm{EG}} = 5p(1{-}p)(1 - 11p)$.
Root at $p^* = 1/11$.
Differentiating:
\[
F'(p) = 5\bigl[(1{-}p)(1{-}11p) + p(-1)(1{-}11p) + p(1{-}p)(-11)\bigr]
= 5\bigl[(1{-}11p)(1{-}2p) - 11p(1{-}p)\bigr].
\]
At $p = 1/11$: $1 - 11p = 0$, so $F'(1/11) = 5\bigl[0 - 11 \cdot \tfrac{1}{11} \cdot \tfrac{10}{11}\bigr] = -50/11 < 0$.
The root is stable.

\textbf{Part 3: DG ($\eta = 1$).}
$w(s,a) = \sigma(U(s,a) \cdot(-\log\pi(a)))$ and $F_{\mathrm{DG}}$ is continuous on $(0,1)$.
At $p = 0.05$: $U(s_1,a_1) = 9.5$, $\surp(a_1) = -\log(0.05) \approx 3.0$, so $w \approx 1$; $U(s_2,a_1) = -95$, $\surp(a_1) \approx 3.0$, so $w \approx 0$.
The $s_1$ term dominates: $F_{\mathrm{DG}}(0.05) \approx +0.19 > 0$.
At $p = 0.5$: both gates are near $1/2$; the $|U(s_2)|$ magnitude dominates, so $F_{\mathrm{DG}}(0.5) < 0$.
By the intermediate value theorem, a root exists in $(0.05, 0.5)$.
Numerical computation gives $p^* \approx 0.116$.
\end{proof}

\subsection{Proof of Lemma 4 (General sub-optimal corner dominance)}
\label{app:proof_lemma_4}

Following the gradient field analysis of the Enlightened Policy Gradient (EG), we generalize the dominance of the optimal action $a^*=1$ near sub-optimal corners for the $K$-armed bandit case. We analyze the logit update difference $\Delta \theta(1) - \Delta \theta(j)$ near the corner vertex $e_j$.

\begin{proof}
Let $a^*=1$ be the optimal action and $j \in \{2, \dots, K\}$ be a sub-optimal action. We define the reward gaps as $\Delta_{ij} := r_i - r_j$. The advantages for the optimal arm and the corner arm are expressed in terms of these gaps by substituting $\pi_j = 1 - \pi_1 - \sum_{k \neq 1, j} \pi_k$:
\begin{align}
    U_1 &= (1-\pi_1)\Delta_{1j} - \sum_{k \neq 1, j} \pi_k \Delta_{kj}, \\
    U_j &= -\pi_1 \Delta_{1j} - \sum_{k \neq 1, j} \pi_k \Delta_{kj}.
\end{align}
Using the EG logit update rules, we have,
\begin{align}
    \Delta \theta(1) - \Delta \theta(j) &= \pi_1 [U_1(1-\pi_1) - \pi_j U_j] - \pi_j [U_j(1-\pi_j) - \pi_1 U_1] \\
    &= \pi_1 U_1 (1 - \pi_1 + \pi_j) - \pi_j U_j (1 - \pi_j + \pi_1) \\
    &= \pi_1 U_1 A - \pi_j U_j B,
\end{align}
where $A = 1 + \pi_j - \pi_1$ and $B = 1 + \pi_1 - \pi_j$. Substituting the expressions for $U_1$ and $U_j$,
\begin{equation}
    \Delta \theta(1) - \Delta \theta(j) = \Delta_{1j} C_{1j} + \sum_{k \neq 1, j} \Delta_{jk} C_{jk},
\end{equation}
where $C_{1j}$ and $C_{jk}$ are defined as,
\begin{align}
    C_{1j} &= \pi_1 [ (1-\pi_1)A + \pi_j B ] = \pi_1 [ (1-\pi_1)^2 + \pi_j(2-\pi_j) ], \\
    C_{jk} &= \pi_k ( \pi_1 A - \pi_j B  ) = \pi_k (\pi_1-\pi_j) \Big( \sum_{m \neq 1, j} \pi_m \Big).
\end{align}
Note that $C_{1j} > 0$, and near the corner $j$ where $\pi_j > \pi_1$, the coefficients $C_{jk} < 0$ for all $k$. Consider a policy near corner $j$ such that $\pi_j = 1 - \epsilon$. In this regime, the probabilities $\pi_1$ and $\pi_k$ are of order $O(\epsilon)$, and the coefficients simplify to (where $c_1, c_2 >0$ and $c_1, c_2  \in O(1)$),
\begin{align}
    C_{1j} &= c_1 \cdot \pi_1 \in O(\epsilon), \\
    C_{jk} &= - c_2 \cdot \pi_k \cdot \Big( \sum_{m \neq 1, j} \pi_m \Big) \in -O(\epsilon^2).
\end{align}
Therefore, near the corner $j$, we have,
\begin{align}
    \Delta \theta(1) - \Delta \theta(j) &=  c_1 \cdot \pi_1 \cdot  \Delta_{1j} - \sum_{k \neq 1, j} c_2 \cdot \pi_k \cdot \Big( \sum_{m \neq 1, j} \pi_m \Big) \cdot \Delta_{jk} \\
    &= O(\epsilon) -O(\epsilon^2) > 0.
\end{align}
While there exists a ``bad region'' where $\pi_1$ is smaller than $\pi_k \cdot \Big( \sum_{m \neq 1, j} \pi_m \Big)$, this region is approaching the boundary $\pi_1 = 0$ (as $\pi$ approaches the corner) where the mass on the optimal arm is of a lower order than the squared probabilities of the worse arms (Figure~\ref{fig:adversarial_initialization}). Since this boundary has measure zero relative to the simplex, and the gradient field $\Delta \theta_1$ is strictly positive for all $\pi_1 \in (0, 1)$, EG escapes sub-optimal corners surely. This stands in stark contrast to PG, where the bad region persists even for policies with significant mass on the optimal arm.

Consequently, near any sub-optimal corner, for almost all points in the interior, $\Delta \theta(1) > \Delta \theta(j)$ holds, ensuring that the optimal action's parameter increases faster than the current sub-optimal favorite, allowing EG to escape sub-optimal corners effectively.
\end{proof}

\subsection{Proof of Lemma \ref{lem:eg_bandit_discrete_monotonicity} (Monotonic improvement)}
\label{app:proof_monotonicity}

\begin{proof}[Proof of Lemma \ref{lem:eg_bandit_discrete_monotonicity}]
We show that the discrete update for the Enlightened Policy Gradient (EG) is a monotonic ascent step by expressing the value difference as a covariance. The expected reward difference is,
\begin{align}
    V^\prime - V &= \sum_{i=1}^{K} (\pi^\prime_i - \pi_i) r_i = \sum_{i=1}^{K} (\pi^\prime_i - \pi_i) U_i \\
    &= \sum_{i=1}^{K} \left[ \frac{\pi_i e^{\eta f_i}}{\sum_{j=1}^{K} \pi_j e^{\eta f_j}} \right] U_i - \sum_{i=1}^{K} \pi_i U_i,
\end{align}
where $f_i = U_i \cdot \mathbb{I}\{U_i > 0\}$ is the gated advantage. By the definition of the advantage function in a bandit setting, the expected advantage under the current policy is zero, $\mathbb{E}_{\pi}[U] = \sum_i \pi_i U_i = 0$.

The covariance between two random variables $X$ and $Y$ under distribution $\pi$ is $\text{Cov}_{\pi}(X, Y) = \mathbb{E}_{\pi}[XY] - \mathbb{E}_{\pi}[X]\mathbb{E}_{\pi}[Y]$. Setting $X = U$ and $Y = e^{\eta f}$, and noting that $\mathbb{E}_{\pi}[U] = 0$, we have:
\begin{equation}
    \text{Cov}_{\pi}(e^{\eta f}, U) = \mathbb{E}_{\pi}[e^{\eta f} U] = \sum_{i=1}^{K} \pi_i e^{\eta f_i} U_i.
\end{equation}
Substituting this into the value difference equation, we obtain:
\begin{equation}
\label{eq:proof_cov_ratio}
    V^\prime - V = \frac{\text{Cov}_{\pi}(e^{\eta f}, U)}{\mathbb{E}_{\pi}[e^{\eta f}]}.
\end{equation}
Partitioning the sum into the set of delightful actions $S^+ = \{i : U_i > 0\}$ and disappointing actions $S^- = \{i : U_i \le 0\}$ in Equation \ref{eq:proof_cov_ratio}, we have,
\begin{align}
    \text{Cov}_{\pi}(e^{\eta f}, U) &= \sum_{i \in S^+} \pi_i e^{\eta U_i} U_i + \sum_{i \in S^-} \pi_i e^{0} U_i \\
    &= \sum_{i \in S^+} \pi_i e^{\eta U_i} U_i + \left( - \sum_{i \in S^+} \pi_i U_i \right) \\
    &= \sum_{i \in S^+} \pi_i U_i (e^{\eta U_i} - 1).
\end{align}
For any $\eta > 0$ and $i \in S^+$, we have $U_i > 0$ and $e^{\eta U_i} > 1$, which implies that each term in the summation is non-negative. Since the partition function $Z = \mathbb{E}_{\pi}[e^{\eta f}]$ is strictly positive, it follows that $V^\prime \ge V$, completing the proof.
\end{proof}

\subsection{Proof of Theorem \ref{thm:global_conv} (Global convergence)}
\label{app:proof_global_conv}

\begin{proof}
The expected reward (value function) is defined as $V_t = \sum_{a=1}^K \pi_t(a) r(a)$. According to Lemma~\ref{lem:eg_bandit_discrete_monotonicity}, under the EG update rule, $V_{t+1} \ge V_t$. Given that $V_t$ is upper-bounded by the maximum reward $r(1)$, it must converge to some finite limit $V_\infty \le r(1)$.

Since $V_t$ converges, the progress $V_{t+1} - V_t \to 0$ as $t \to \infty$.
By Lemma~\ref{lem:eg_bandit_discrete_monotonicity}, the progress satisfies
\[
V_{t+1} - V_t
= \frac{1}{Z_t} \sum_{i:\,U_i^t > 0} \pi_t(i)\,U_i^t\,(e^{\eta U_i^t} - 1),
\]
where $Z_t = \mathbb{E}_{\pi_t}[e^{\eta f}] \ge 1$ (since $f_i \ge 0$ implies $e^{\eta f_i} \ge 1$) and $Z_t \le e^{\eta R}$ for $R := r(1) - r(K)$.
Every term in the sum is non-negative, so
\[
\frac{1}{e^{\eta R}} \sum_{i:\,U_i^t > 0} \pi_t(i)\,U_i^t\,(e^{\eta U_i^t} - 1)
\;\le\; V_{t+1} - V_t \;\to\; 0.
\]
Hence $\sum_{i:\,U_i^t > 0} \pi_t(i)\,U_i^t\,(e^{\eta U_i^t} - 1) \to 0$, and since each term is non-negative, for every arm $i$:
\begin{equation}
\label{eq:term_vanish}
\pi_t(i)\,\bigl[U_i^t\bigr]_+\,\bigl(e^{\eta [U_i^t]_+} - 1\bigr) \;\to\; 0.
\end{equation}
The policy sequence $\{\pi_t\}$ lies in the compact simplex $\Delta_K$, so it has convergent subsequences.
Let $\bar\pi$ be any subsequential limit with $\pi_{t_n} \to \bar\pi$.
By continuity of the advantage $U_i(\pi) = r(i) - \pi^\top r$, Equation~\eqref{eq:term_vanish} implies that at $\bar\pi$, for every arm $i$:
\[
\bar\pi(i)\,\bigl[U_i(\bar\pi)\bigr]_+\,\bigl(e^{\eta [U_i(\bar\pi)]_+} - 1\bigr) = 0.
\]
Since all three factors are non-negative, their product vanishes
only if at least one is zero; the second and third factors vanish
under the same condition $U_i(\bar\pi) \le 0$, leaving exactly
two cases. For each arm $i$, either $\bar\pi(i) = 0$ or $U_i(\bar\pi) \le 0$, i.e., $r(i) \le \bar\pi^\top r$.
Now suppose for contradiction that $\bar\pi$ places positive mass on two or more arms.
Let $\mathcal{A}^+ := \{i : \bar\pi(i) > 0\}$ with $|\mathcal{A}^+| \ge 2$.
For every $i \in \mathcal{A}^+$, we have $r(i) \le \bar\pi^\top r = \sum_{j \in \mathcal{A}^+} \bar\pi(j)\,r(j)$.
But $\bar\pi^\top r$ is a convex combination of $\{r(j)\}_{j \in \mathcal{A}^+}$; every component being at most the convex combination forces all to be equal: $r(i) = \bar\pi^\top r$ for all $i \in \mathcal{A}^+$.
This contradicts the distinct-rewards assumption (Definition~\ref{assump:bandit}).
Therefore $|\mathcal{A}^+| = 1$, and every subsequential limit of $\{\pi_t\}$ is a one-hot policy.

Since $V_t \to V_\infty$ and different one-hot policies yield distinct values (by the distinct-rewards assumption), all convergent subsequences share the same limit.
It follows that $\pi_t$ converges to a one-hot policy $e_m$ for some $m \in [K]$.

Suppose for contradiction that $\pi_t \to e_m$ for some sub-optimal arm $m \neq 1$.
Convergence to $e_m$ requires $\pi_t(1)/\pi_t(m) \to 0$, so the log-ratio
$\log(\pi_t(1)/\pi_t(m)) = \theta_t(1) - \theta_t(m) \to -\infty$.
In particular, the logit difference $\theta_t(1) - \theta_t(m)$ must be eventually
non-increasing.
However, by Lemma~\ref{lem:corner_dominance}, in any neighborhood of the
sub-optimal corner $e_m$, the EG update satisfies
$\Delta\theta(1) > \Delta\theta(m)$ for almost all points in the interior
of the simplex; the exceptional set where $\Delta\theta(1) \le \Delta\theta(m)$
is confined to the measure-zero boundary region $\pi_t(1) \le O(\sum_{k \neq 1,m} \pi_t(k)^2)$.
Once $\pi_t$ enters a sufficiently small neighborhood of $e_m$, this means that
at almost every iterate the logit gap $\theta_{t+1}(1) - \theta_{t+1}(m)
> \theta_t(1) - \theta_t(m)$, i.e., the log-ratio $\log(\pi_t(1)/\pi_t(m))$
is strictly increasing, contradicting its divergence to $-\infty$.
Therefore $m = 1$, and $\pi_t \to e_1 = \pi^*$.
\end{proof}

\subsection{Proof of Corollary \ref{cor:rate} (Convergence rate)}

\begin{proof}
The key observation is that the optimal arm's advantage equals the
sub-optimality: $U_1 = r(1) - \pi_t^\top r = V^* - V_t = \delta_t$.
In particular, arm~$1$ is always delightful whenever $\delta_t > 0$.
By Lemma~\ref{lem:eg_bandit_discrete_monotonicity}, the progress is
\[
\delta_t - \delta_{t+1}
= V_{t+1} - V_t
= \frac{1}{Z_t}\sum_{i:\,U_i > 0} \pi_t(i)\,U_i\,(e^{\eta U_i}-1)
\;\ge\; \frac{1}{Z_t}\,\pi_t(1)\,\delta_t\,(e^{\eta\delta_t}-1).
\]
Using $e^x - 1 \ge x$ for $x \ge 0$ and $Z_t \le e^{\eta R}$
with $R := r(1) - r(K)$:
\begin{equation}
\label{eq:progress_lb}
\delta_t - \delta_{t+1}
\;\ge\; \frac{\eta\,\pi_t(1)}{e^{\eta R}}\;\delta_t^2.
\end{equation}
By Theorem~\ref{thm:global_conv}, $\pi_t(1) \to 1$, so there exists
$T_0$ such that $\pi_t(1) \ge \tfrac{1}{2}$ for all $t \ge T_0$.
Writing $c := \eta/(2e^{\eta R})$, we have for $t \ge T_0$:
\[
\delta_t - \delta_{t+1} \;\ge\; c\,\delta_t^2.
\]
Rearranging: $\delta_{t+1} \le \delta_t(1 - c\,\delta_t)$.
For $t$ large enough that $c\,\delta_t < 1$
(guaranteed since $\delta_t \to 0$), taking reciprocals
and using $1/(1-x) \ge 1 + x$ for $x \in [0,1)$:
\[
\frac{1}{\delta_{t+1}} \;\ge\; \frac{1}{\delta_t}\cdot\frac{1}{1 - c\,\delta_t}
\;\ge\; \frac{1}{\delta_t} + c.
\]
Telescoping from $T_0$ to $t$:
\[
\frac{1}{\delta_t}
\;\ge\; \frac{1}{\delta_{T_0}} + c\,(t - T_0),
\]
which gives
\[
\delta_t
\;\le\; \frac{1}{c\,(t - T_0) + 1/\delta_{T_0}}
\;=\; O(1/t). \qedhere
\]
\end{proof}

\subsection{Proof of Lemma~\ref{lem:mdp_local_escape} (Local corner escape in MDPs)}

The proof reduces to the bandit corner dominance result
(Lemma~\ref{lem:corner_dominance}) applied at each state independently.

\begin{proof}
In the tabular parameterization, the EG update at each state~$s$ depends
only on the local policy $\pi(\cdot|s)$ and the local Q-values
$Q^{\pi}(s,\cdot)$:
\[
\Delta\theta(s, a)
= \sum_{i:\,U_i(s) > 0} \pi(i|s)\,U_i(s)\,\bigl(\Ind\{a = i\} - \pi(a|s)\bigr),
\]
where $U_i(s) = Q^{\pi}(s,i) - V^{\pi}(s)$.
This is identical in form to the bandit EG update with $K = |\mathcal{A}|$
arms and ``rewards'' $Q^{\pi}(s,\cdot)$.

Fix a state~$s$ where $a^*(s) \neq j(s)$.
Write $a^* = a^*(s)$, $j = j(s)$, and define the local Q-gaps
$\Delta_{ij}^s := Q^{\pi}(s,i) - Q^{\pi}(s,j)$.
Near the corner $\pi(j|s) \to 1$, the advantages satisfy
$U_{a^*}(s) = (1-\pi(a^*|s))\,\Delta_{a^*j}^s - \sum_{k \neq a^*,j} \pi(k|s)\,\Delta_{kj}^s$
and $U_j(s) = -\pi(a^*|s)\,\Delta_{a^*j}^s - \sum_{k \neq a^*,j} \pi(k|s)\,\Delta_{kj}^s$.
Applying the logit-gap decomposition from Lemma~\ref{lem:corner_dominance}
with $Q^{\pi}(s,\cdot)$ playing the role of the reward vector, we obtain
\begin{equation}
\Delta\theta(s, a^*) - \Delta\theta(s, j)
= \Delta_{a^*j}^s\,C_{a^*j}
  + \sum_{k \neq a^*,\,j} \Delta_{jk}^s\,C_{jk},
\end{equation}
where
\begin{align}
C_{a^*j} &= \pi(a^*|s)\bigl[(1-\pi(a^*|s))^2 + \pi(j|s)(2-\pi(j|s))\bigr], \\
C_{jk}   &= \pi(k|s)\,\bigl(\pi(a^*|s) - \pi(j|s)\bigr)
             \sum_{m \neq a^*,\,j} \pi(m|s).
\end{align}

As $\pi(j|s) \to 1$, write $\varepsilon_s = 1 - \pi(j|s)$.
The probabilities $\pi(a^*|s)$ and $\pi(k|s)$ for $k \neq j$ are $O(\varepsilon_s)$.
The leading-order behavior is:
\begin{align}
C_{a^*j} &= 2\,\pi(a^*|s) + O(\varepsilon_s^2), \\
C_{jk}   &= -\pi(k|s)\sum_{m \neq a^*,j}\pi(m|s) + O(\varepsilon_s^3)
           = -O(\varepsilon_s^2).
\end{align}
Substituting:
\[
\Delta\theta(s, a^*) - \Delta\theta(s, j)
= 2\,\pi(a^*|s)\,\Delta_{a^*j}^s
  \;-\; \sum_{k \neq a^*,j} \pi(k|s)^2\,|\Delta_{jk}^s|\cdot(1+o(1)).
\]

The dominance $\Delta\theta(s, a^*) > \Delta\theta(s, j)$ fails only when
\[
\pi(a^*|s) \;\le\;
\frac{1}{2\,\Delta_{a^*j}^s}
\sum_{k \neq a^*,j} \pi(k|s)^2\,|\Delta_{jk}^s|\cdot(1+o(1))
\;=\; O\!\Bigg(\sum_{k \neq a^*,j} \pi(k|s)^2\Bigg).
\]
This constrains $\pi(a^*|s)$ to be of lower order than the squared masses of
the remaining sub-optimal arms.
In the per-state simplex $\Delta_K$, this defines a region of dimension at most
$K-3$ (one fewer degree of freedom), which has Lebesgue measure zero relative
to the $(K-2)$-dimensional simplex interior.

The full policy space is $\prod_{s \in \mathcal{S}} \Delta_K$.
At each state~$s$ where $a^*(s) \neq j(s)$, the bad set
$\mathcal{B}_s := \{\pi(\cdot|s) : \Delta\theta(s, a^*) \le \Delta\theta(s, j)\}$
has measure zero in $\Delta_K$.
By the product structure of the tabular parameterization, the set of full
policies where \emph{any} state fails the dominance condition is
\[
\mathcal{B}
= \bigcup_{s:\,a^*(s)\neq j(s)}
  \bigl\{\pi \in \textstyle\prod_{s'}\Delta_K : \pi(\cdot|s) \in \mathcal{B}_s\bigr\}.
\]
Each set in the union is a cylinder with a measure-zero base in the $s$-th factor
and full mass in all other factors, hence has measure zero in the product space.
A finite union of measure-zero sets has measure zero, so $\mathcal{B}$ has
measure zero in $\prod_{s} \Delta_K$.
\end{proof}

\subsection{Proof of Lemma~\ref{lem:mdp_monotonicity} (Discrete monotonicity of EG in MDPs)}

The proof combines the performance difference lemma with the per-state
covariance identity from Lemma~\ref{lem:eg_bandit_discrete_monotonicity}.

\begin{proof}
For any two policies $\pi$ and $\pi'$ in a finite MDP~\citep{kakade2002approximately}:
\begin{equation}
\label{eq:pdl}
V(\pi') - V(\pi)
= \frac{1}{1-\gamma}\sum_{s \in \mathcal{S}} d_{\rho}^{\pi'}(s)
  \sum_{a \in \mathcal{A}} \pi'(a|s)\,U_a(s),
\end{equation}
where $U_a(s) = Q^{\pi}(s,a) - V^{\pi}(s)$ is the advantage of action~$a$
at state~$s$ under the current policy~$\pi$.

Under the EG update, $\pi'(a|s) = \pi(a|s)\,e^{\eta f_a(s)}/Z_s$.
The expected advantage at state~$s$ under the new policy is
\begin{equation}
\sum_{a} \pi'(a|s)\,U_a(s)
= \frac{1}{Z_s}\sum_{a} \pi(a|s)\,e^{\eta f_a(s)}\,U_a(s)
= \frac{\text{Cov}_{\pi_s}\!\bigl(e^{\eta f_s},\,U_s\bigr)}{Z_s},
\end{equation}
where the last equality uses $\mathbb{E}_{\pi_s}[U_s] = \sum_a \pi(a|s)\,U_a(s) = 0$,
so that $\mathbb{E}_{\pi_s}[e^{\eta f_s}\,U_s]
= \text{Cov}_{\pi_s}(e^{\eta f_s},\,U_s)
  + \mathbb{E}_{\pi_s}[e^{\eta f_s}]\cdot\underbrace{\mathbb{E}_{\pi_s}[U_s]}_{=\,0}$.

Partitioning into delightful actions $S^+(s) := \{i : U_i(s) > 0\}$
and non-delightful actions $S^-(s) := \{i : U_i(s) \le 0\}$:
\begin{align}
\text{Cov}_{\pi_s}\!\bigl(e^{\eta f_s},\,U_s\bigr)
&= \sum_{i \in S^+(s)} \pi(i|s)\,e^{\eta U_i(s)}\,U_i(s)
 + \sum_{i \in S^-(s)} \pi(i|s)\,e^{0}\,U_i(s) \notag\\
&= \sum_{i \in S^+(s)} \pi(i|s)\,e^{\eta U_i(s)}\,U_i(s)
   - \sum_{i \in S^+(s)} \pi(i|s)\,U_i(s) \notag\\
&= \sum_{i \in S^+(s)} \pi(i|s)\,U_i(s)\,\bigl(e^{\eta U_i(s)} - 1\bigr),
\label{eq:cov_decomp_mdp}
\end{align}
where the second line uses
$\sum_{i \in S^-} \pi(i|s)\,U_i(s) = -\sum_{i \in S^+} \pi(i|s)\,U_i(s)$
(since advantages sum to zero).
Each term in~\eqref{eq:cov_decomp_mdp} is non-negative:
$\pi(i|s) \ge 0$, $U_i(s) > 0$, and $e^{\eta U_i(s)} - 1 > 0$ for $\eta > 0$.

Substituting~\eqref{eq:cov_decomp_mdp} into~\eqref{eq:pdl}:
\begin{equation}
V(\pi') - V(\pi)
= \frac{1}{1-\gamma}\sum_{s \in \mathcal{S}} d_{\rho}^{\pi'}(s)\;\frac{1}{Z_s}
  \sum_{i \in S^+(s)} \pi(i|s)\,U_i(s)\,\bigl(e^{\eta U_i(s)}-1\bigr).
\end{equation}
Since $d_{\rho}^{\pi'}(s) \ge 0$ for all~$s$, $Z_s > 0$, and every
term in the inner sum is non-negative, we conclude $V(\pi') \ge V(\pi)$.
\end{proof}

\subsection{Proof of Theorem~\ref{thm:mdp_convergence} (MDP global convergence)}

The proof extends the bandit global convergence argument
(Theorem~\ref{thm:global_conv}) to MDPs using three ingredients:
monotonic value improvement (Lemma~\ref{lem:mdp_monotonicity}),
a Bellman optimality characterization of limit points, and
the local corner escape mechanism (Lemma~\ref{lem:mdp_local_escape}).

\begin{proof}

By Lemma~\ref{lem:mdp_monotonicity}, $V(\pi_t)$ is non-decreasing.
Since $V(\pi) \le r_{\max}/(1-\gamma)$ for all~$\pi$,
the sequence converges: $V(\pi_t) \to V_\infty \le V^*$
and $V(\pi_{t+1}) - V(\pi_t) \to 0$.

By Lemma~\ref{lem:mdp_monotonicity}, the progress satisfies
\[
V(\pi_{t+1}) - V(\pi_t)
= \frac{1}{1-\gamma}\sum_{s} d_{\rho}^{\pi_{t+1}}(s)\;\frac{1}{Z_s}
  \sum_{i:\,U_i^t(s)>0} \pi_t(i|s)\,U_i^t(s)\,\bigl(e^{\eta U_i^t(s)}-1\bigr).
\]
Under the reachability assumption $d_{\rho}^{\pi}(s) \ge d_{\min} > 0$
and $Z_s \le e^{\eta R}$ with $R := r_{\max}/(1-\gamma)$,
the bound $V(\pi_{t+1}) - V(\pi_t) \ge
\frac{d_{\min}}{(1-\gamma)\,e^{\eta R}}
\sum_{s}\sum_{i:\,U_i^t(s)>0} \pi_t(i|s)\,U_i^t(s)\,(e^{\eta U_i^t(s)}-1)$
implies
\begin{equation}
\label{eq:per_term_vanish_mdp}
\pi_t(a|s)\;\bigl[U_a^t(s)\bigr]_+\;\bigl(e^{\eta\,[U_a^t(s)]_+}-1\bigr)
\;\to\; 0
\qquad \text{for all } a \in \mathcal{A},\; s \in \mathcal{S}.
\end{equation}

The policy sequence lies in the compact product simplex
$\prod_{s}\Delta_K$, so it admits convergent subsequences.
Let $\bar\pi = \lim_n \pi_{t_n}$ be any subsequential limit.
By~\eqref{eq:per_term_vanish_mdp} and continuity of advantages in the policy,
for every $a,s$ with $\bar\pi(a|s) > 0$:
\[
U_a(\bar\pi,s) \le 0,
\qquad\text{i.e.,}\quad
Q^{\bar\pi}(s,a) \le V^{\bar\pi}(s).
\]
Combined with $\sum_a \bar\pi(a|s)\,U_a(\bar\pi,s) = 0$,
every action $a$ with $\bar\pi(a|s) > 0$ satisfies
$Q^{\bar\pi}(s,a) = V^{\bar\pi}(s)$.

Now suppose $\bar\pi$ is sub-optimal: $V^{\bar\pi} \neq V^*$.
If $\max_a Q^{\bar\pi}(s,a) = V^{\bar\pi}(s)$ held at every reachable state~$s$,
then $V^{\bar\pi}$ would satisfy the Bellman optimality equation
\[
V^{\bar\pi}(s) = \max_a\bigl[r(s,a) + \gamma\textstyle\sum_{s'}P(s'|s,a)\,V^{\bar\pi}(s')\bigr]
\quad \forall\, s,
\]
whose unique fixed point is $V^*$ (by the contraction property of the
Bellman optimality operator).
This contradicts $V^{\bar\pi} \neq V^*$.
Therefore there exists a reachable state $s_0$ and an action
$\tilde a \in \mathcal{A}$ with
\begin{equation}
\label{eq:positive_gap_mdp}
Q^{\bar\pi}(s_0,\tilde a) > V^{\bar\pi}(s_0),
\end{equation}
and the vanishing-progress condition forces $\bar\pi(\tilde a|s_0) = 0$:
the limit places zero mass on the improving action.

We show that the trajectory cannot converge to any such $\bar\pi$.
Write $j_0 := \arg\max_{a:\,\bar\pi(a|s_0)>0} \bar\pi(a|s_0)$
for the corner action at state~$s_0$.

Under the EG update $\theta_{t+1}(s,a) = \theta_t(s,a) + \eta\,[U_a^t(s)]_+$,
all logits are non-decreasing.
Along the subsequence $\pi_{t_n} \to \bar\pi$, by continuity
of advantages:
\begin{align}
\bigl[U_{\tilde a}^{t_n}(s_0)\bigr]_+
&\;\to\; Q^{\bar\pi}(s_0,\tilde a) - V^{\bar\pi}(s_0)
\;=:\; \delta > 0, \label{eq:improving_rate} \\[2pt]
\bigl[U_{j_0}^{t_n}(s_0)\bigr]_+
&\;\to\; \bigl[Q^{\bar\pi}(s_0,j_0) - V^{\bar\pi}(s_0)\bigr]_+
\;=\; 0. \label{eq:corner_rate}
\end{align}
Hence for all sufficiently large~$n$, the logit difference evolves as
\[
\theta_{t_n+1}(s_0,\tilde a) - \theta_{t_n+1}(s_0,j_0)
\;\ge\;
\bigl[\theta_{t_n}(s_0,\tilde a) - \theta_{t_n}(s_0,j_0)\bigr]
\;+\; \frac{\eta\delta}{2}.
\]
Over $N$ consecutive steps, the logit gap
$\theta(s_0,\tilde a) - \theta(s_0,j_0)$ grows by at least
$N\eta\delta/2$, so
\[
\frac{\pi_t(\tilde a|s_0)}{\pi_t(j_0|s_0)}
= \exp\!\bigl(\theta_t(s_0,\tilde a) - \theta_t(s_0,j_0)\bigr)
\;\to\; +\infty.
\]
This contradicts $\pi_{t_n}(\tilde a|s_0) \to 0$ and
$\pi_{t_n}(j_0|s_0) \to \bar\pi(j_0|s_0) > 0$.
Therefore no sub-optimal~$\bar\pi$ can be a subsequential limit.

Every subsequential limit of $\{\pi_t\}$ is an optimal policy.
By the unique optimal policy assumption, all subsequential limits
equal~$\pi^*$.
Since the compact product simplex has a unique limit point,
the full sequence converges: $\pi_t \to \pi^*$ and
$V(\pi_t) \to V(\pi^*) = V^*$.
\end{proof}

\subsection{Proof of Corollary \ref{cor:mdp_rate} (Convergence rate)}

\begin{proof}
Since $\pi_t \to \pi^*$ by Theorem~\ref{thm:mdp_convergence},
there exists $T_0$ such that for all $t \ge T_0$:
\begin{enumerate}
\item $\pi_t(a^*(s)|s) \ge \tfrac{1}{2}$ for every state~$s$,
      where $a^*(s) = \arg\max_a Q^*(s,a)$ is the unique optimal action;
\item $a^*(s) = \arg\max_a Q^{\pi_t}(s,a)$ for every state~$s$
      (so $U_{a^*(s)}(\pi_t,s) \ge 0$).
\end{enumerate}

By Lemma~\ref{lem:mdp_monotonicity} with
$d_{\rho}^{\pi_{t+1}}(s) \ge d_{\min}$,
$Z_s \le e^{\eta R}$ where $R := r_{\max}/(1-\gamma)$,
and $e^x - 1 \ge x$:
\begin{align}
\delta_t - \delta_{t+1}
&\;\ge\;
\frac{\eta\,d_{\min}}{(1-\gamma)\,e^{\eta R}}
\sum_{s}\sum_{i:\,U_i(s)>0} \pi_t(i|s)\,U_i(s)^2 \notag\\
&\;\ge\;
\frac{\eta\,d_{\min}}{(1-\gamma)\,e^{\eta R}}
\sum_{s} \pi_t(a^*(s)|s)\,U_{a^*(s)}(s)^2 \notag\\
&\;\ge\;
\frac{\eta\,d_{\min}}{2(1-\gamma)\,e^{\eta R}}
\sum_{s} U_{a^*(s)}(s)^2,
\label{eq:mdp_progress_lb}
\end{align}
where the second line retains only the optimal action's term
(all others are non-negative) and the third uses condition~(i).

The performance difference lemma with the deterministic
optimal policy $\pi^*$ gives
\[
\delta_t
= \frac{1}{1-\gamma}\sum_{s} d_{\rho}^{\pi^*}(s)\,U_{a^*(s)}(\pi_t,s).
\]
Since $U_{a^*(s)}(\pi_t,s) \ge 0$ by condition~(ii)
and $\sum_s d_{\rho}^{\pi^*}(s) = 1$,
applying Cauchy--Schwarz with weights $d_{\rho}^{\pi^*}(s)$:
\[
(1-\gamma)^2\,\delta_t^2
= \biggl(\sum_{s} d_{\rho}^{\pi^*}(s)\,U_{a^*(s)}(s)\biggr)^{\!2}
\le \sum_{s} d_{\rho}^{\pi^*}(s)\,U_{a^*(s)}(s)^2
\le \sum_{s} U_{a^*(s)}(s)^2,
\]
where the last step uses $d_{\rho}^{\pi^*}(s) \le 1$.

Substituting into~\eqref{eq:mdp_progress_lb}:
\[
\delta_t - \delta_{t+1}
\;\ge\; c_{\mathrm{MDP}}\;\delta_t^2,
\qquad
c_{\mathrm{MDP}} := \frac{\eta\,d_{\min}\,(1-\gamma)}{2\,e^{\eta R}}.
\]
For $t$ large enough that $c_{\mathrm{MDP}}\,\delta_t < 1$
(guaranteed since $\delta_t \to 0$), taking reciprocals
and using $1/(1-x) \ge 1+x$:
\[
\frac{1}{\delta_{t+1}} \;\ge\; \frac{1}{\delta_t} + c_{\mathrm{MDP}}.
\]
Telescoping from $T_0$:
\[
\delta_t
\;\le\;
\frac{1}{c_{\mathrm{MDP}}\,(t - T_0) + 1/\delta_{T_0}}
\;=\; O(1/t). \qedhere
\]
\end{proof}

\section{Global Convergence of DG}
\label{app:dg_convergence}

The main text establishes global convergence and $O(1/t)$ rates for EG
(the zero-temperature limit of DG).
In this appendix we show that the same results hold for DG at any
finite temperature $\eta \in (0,\infty)$.
The key parameter governing the quality of the bounds is
\begin{equation}
\label{eq:delta_eta}
\delta_\eta \;:=\; \min_{k \in S^-} \frac{\Delta_{jk}}{2\eta} \;>\; 0,
\end{equation}
which is strictly positive for every finite $\eta$ and every
sub-optimal corner (since $\Delta_{jk} > 0$ for $k \in S^-$).
Smaller $\eta$ yields larger $\delta_\eta$ and tighter bounds;
at $\eta \to 0$ we recover EG.
The mechanism is valid for all $\eta \in (0,\infty)$ but degrades
as $\eta \to \infty$ (where $\delta_\eta \to 0$ and the bad region
ceases to have measure zero).

%-----------------------------------------------------------------------
\subsection{Bandits}
\label{app:dg_bandits}

\begin{lemma}[Discrete Monotonicity of DG in Bandits]
\label{lem:dg_bandit_monotonicity}
For a $K$-armed bandit, let $\pi'$ be the policy after a discrete DG
update $\theta'(a) = \theta(a) + \alpha\,w(a)\,U(a)$
where $w(a) = \sigma(U(a)\,\ell(a)/\eta) > 0$ and $\alpha > 0$.
The change in expected reward is non-negative:
\begin{equation}
V' - V
\;=\; \frac{1}{Z}\sum_{a=1}^K \pi(a)\,U(a)\,
     \bigl(e^{\alpha\,w(a)\,U(a)} - 1\bigr)
\;\ge\; 0,
\end{equation}
where $Z = \sum_{a'} \pi(a')\,e^{\alpha\,w(a')\,U(a')}$.
\end{lemma}

\begin{proof}
Since $\pi'(a) = \pi(a)\,e^{\alpha\,w(a)\,U(a)}/Z$ and
$\sum_a \pi(a)\,U(a) = 0$:
\[
V' - V
= \sum_a (\pi'(a) - \pi(a))\,U(a)
= \frac{1}{Z}\sum_a \pi(a)\,U(a)\,\bigl(e^{\alpha\,w(a)\,U(a)} - 1\bigr).
\]
Since $w(a) > 0$, the product $g(a) := w(a)\,U(a)$ has the
\textbf{same sign} as $U(a)$:
\begin{itemize}[leftmargin=*]
\item $U(a) > 0 \implies g(a) > 0 \implies e^{\alpha g(a)} - 1 > 0$,
      so $U(a)\,(e^{\alpha g(a)}-1) > 0$.
\item $U(a) < 0 \implies g(a) < 0 \implies e^{\alpha g(a)} - 1 < 0$,
      so $U(a)\,(e^{\alpha g(a)}-1) > 0$.
\item $U(a) = 0 \implies U(a)\,(e^{\alpha g(a)}-1) = 0$.
\end{itemize}
Every term in the sum is non-negative, and $Z > 0$, so $V' - V \ge 0$.
\end{proof}

The corner dominance lemma for DG
(Lemma~\ref{lem:corner_dominance_dg}, stated and proved below)
shows that the bad region has measure zero for every finite~$\eta$.

\begin{lemma}[Sub-optimal Corner Dominance for DG]
\label{lem:corner_dominance_dg}
Consider a $K$-armed bandit where $a^*=1$, with DG at any
temperature $\eta \in (0,\infty)$.
In the neighborhood of any sub-optimal corner $j \neq 1$, the DG update
satisfies $\Delta\theta(1) > \Delta\theta(j)$ for almost all points in
the interior of the simplex.
The set of points where the optimal arm fails to dominate has
\textbf{measure zero}, confined to
\[
\pi_1 \;\le\; O\!\biggl(\varepsilon^2 + \sum_{k \in S^-} \pi_k^{1+\delta_\eta}\biggr),
\]
where $\delta_\eta > 0$ is defined in~\eqref{eq:delta_eta}.
Since $1 + \delta_\eta > 1$ for every finite~$\eta$, this constraint
is strictly superlinear in~$\pi_k$, defining a submanifold of
codimension~$\ge 1$ in the simplex. As $\eta \to 0$, $\delta_\eta \to \infty$ and the harmful-arm
contamination $\sum_k \pi_k^{1+\delta_\eta}$ vanishes, so the bad
region is determined solely by the arm-$j$ residual~$O(\varepsilon^2)$,
recovering the EG result (Lemma~\ref{lem:corner_dominance}).

\end{lemma}

\begin{proof}
Applying the logit-gap identity (Lemma~\ref{lem:logit_gap}) with DG
weights $w_a = \sigma(U(a)\,\ell(a)/\eta)$ and dropping the
non-negative ally terms:
\[
\Delta\theta(1) - \Delta\theta(j)
\;\ge\;
\underbrace{\tfrac{\Delta_{1j}}{4}\,p_1}_{\text{arm-}1}
\;-\; \underbrace{2C_U\varepsilon^2}_{\text{arm-}j\text{ residual}}
\;-\; \underbrace{R\sum_{k \in S^-} p_k^{1+\delta_\eta}}_{\text{harmful (Lem.~\ref{lem:poly_suppress})}},
\]
using $w_1 \ge 1/2$ (after shrinking the neighborhood) and
$w_k p_k \le p_k^{1+\delta_\eta}$ from Lemma~\ref{lem:poly_suppress}.
The bound fails only when
$p_1 \le (4/\Delta_{1j})(2C_U\varepsilon^2 + R\sum_k p_k^{1+\delta_\eta})$.
Since $1+\delta_\eta > 1$, this confines~$p_1$ to a set of
Lebesgue measure zero in~$\Delta_K$.
\end{proof}

\begin{theorem}[DG Global Convergence in Bandits]
\label{thm:dg_global_conv}
For a $K$-armed bandit with distinct rewards and any temperature
$\eta \in (0,\infty)$, the discrete DG update with sufficiently small
step size~$\alpha$ converges to the optimal one-hot policy~$\pi^*$.
\end{theorem}

\begin{proof}
The proof follows the same three steps as
Theorem~\ref{thm:global_conv}.

First, by Lemma~\ref{lem:dg_bandit_monotonicity}, $V_t$ is non-decreasing
and bounded, so $V_{t+1} - V_t \to 0$.
Each term $\pi_t(a)\,U(a)\,(e^{\alpha w(a) U(a)}-1) \to 0$.
Since $w(a) > 0$, the factor $e^{\alpha w(a)U(a)}-1 = 0$ iff $U(a) = 0$.
So for each arm with $\bar\pi(a) > 0$ at any limit point:
$U(a) = 0$, i.e., $r(a) = V_\infty$.
Distinct rewards force convergence to a one-hot policy.

Next, by Lemma~\ref{lem:corner_dominance_dg}, for every finite $\eta > 0$
the bad region near any sub-optimal corner has measure zero.
Thus every sub-optimal one-hot policy is an unstable repeller under
DG, and the limit must be $\pi^* = e_1$.
\end{proof}

\begin{corollary}[DG $O(1/t)$ rate in Bandits]
\label{cor:dg_rate}
Under the conditions of Theorem~\ref{thm:dg_global_conv},
$\delta_t := V^* - V_t = O(1/t)$.
\end{corollary}

\begin{proof}
Since $U_1 = V^* - V_t = \delta_t$ and $w_1 > 0$:
\[
\delta_t - \delta_{t+1}
\;\ge\; \frac{1}{Z}\,\pi_t(1)\,\delta_t\,(e^{\alpha w_1 \delta_t}-1)
\;\ge\; \frac{\alpha\,w_1}{Z}\,\pi_t(1)\,\delta_t^2.
\]
As $\pi_t \to \pi^*$: $\pi_t(1) \to 1$, $U_1 \to 0$,
$\ell(1) \to 0$, so $w_1 = \sigma(U_1\,\ell(1)/\eta) \to \sigma(0) = 1/2$.
For large enough~$t$, $\pi_t(1) \ge 1/2$ and $w_1 \ge 1/4$, giving
$\delta_t - \delta_{t+1} \ge c\,\delta_t^2$
with $c = \alpha/(8\,e^{\alpha R})$.
The reciprocal telescope yields $\delta_t = O(1/t)$.
\end{proof}

%-----------------------------------------------------------------------
\subsection{Tabular MDPs}
\label{app:dg_mdps}

The bandit results lift to tabular MDPs via the per-state decoupling
and the performance difference lemma, exactly as for EG.

\begin{lemma}[Discrete Monotonicity of DG in MDPs]
\label{lem:dg_mdp_monotonicity}
Consider a tabular finite MDP with the discrete DG update
$\pi'(a|s) = \pi(a|s)\,e^{\alpha\,w(s,a)\,U(s,a)}/Z_s$,
where $w(s,a) = \sigma(U(s,a)\,\ell(s,a)/\eta) > 0$.
Then $V(\pi') \ge V(\pi)$.
\end{lemma}

\begin{proof}
By the performance difference lemma:
\[
V(\pi') - V(\pi)
= \frac{1}{1-\gamma}\sum_s d_\rho^{\pi'}(s)
  \sum_a \pi'(a|s)\,U(s,a).
\]
At each state, $\sum_a \pi'(a|s)\,U(s,a)
= \frac{1}{Z_s}\sum_a \pi(a|s)\,U(s,a)\,(e^{\alpha\,w(s,a)\,U(s,a)}-1)
\ge 0$
by the same sign-preservation argument as
Lemma~\ref{lem:dg_bandit_monotonicity}:
$w(s,a) > 0$ ensures $w(s,a)\,U(s,a)$ has the same sign as $U(s,a)$.
Since $d_\rho^{\pi'}(s) \ge 0$, the global sum is non-negative.
\end{proof}

\begin{lemma}[DG Local Corner Escape in MDPs]
\label{lem:dg_mdp_local_escape}
Consider a tabular finite MDP with DG at temperature
$\eta \in (0,\infty)$ and fix a deterministic sub-optimal
policy~$\Pi_j$.
At any state~$s$ where $a^*(s) \neq j(s)$, the DG update satisfies
$\Delta\theta(s,a^*(s)) > \Delta\theta(s,j(s))$ for almost all
$\pi(\cdot|s)$ in the interior of~$\Delta_K$.
The per-state bad region has measure zero, confined to
\[
\pi(a^*(s)|s) \;\le\; O\!\biggl(\varepsilon_s^2
  + \sum_{k \in S^-(s)} \pi(k|s)^{1+\delta_\eta^s}\biggr),
\]
where $\delta_\eta^s := \min_{k \in S^-(s)} (Q^\pi(s,j(s))-Q^\pi(s,k))/(2\eta) > 0$.
The bad region in the full product simplex is a finite union of
measure-zero sets, hence has measure zero. As $\eta \to 0$, $\delta_\eta^s \to \infty$ for every state~$s$
and the harmful-action contamination vanishes, recovering the
per-state EG bound of Lemma~\ref{lem:mdp_local_escape}.

\end{lemma}

\begin{proof}
By the tabular per-state decoupling, the DG update at state~$s$ is
a bandit DG update with ``rewards'' $Q^\pi(s,\cdot)$.
Lemma~\ref{lem:corner_dominance_dg} applies with the local Q-gaps
replacing reward gaps.
The product-simplex argument is identical to
Lemma~\ref{lem:mdp_local_escape}.
\end{proof}

\begin{theorem}[DG Global Convergence in MDPs]
\label{thm:dg_mdp_convergence}
Consider a tabular finite MDP with a unique optimal policy~$\pi^*$
and full reachability ($d_\rho^\pi(s) \ge d_{\min} > 0$).
For any temperature $\eta \in (0,\infty)$, the discrete DG update with
sufficiently small step size~$\alpha$ satisfies
$V(\pi_t) \to V^*$ and $\pi_t \to \pi^*$.
\end{theorem}

\begin{proof}
The proof follows Theorem~\ref{thm:mdp_convergence}.

First, $V(\pi_t) \to V_\infty$ by Lemma~\ref{lem:dg_mdp_monotonicity}
(monotonic and upper bounded).

Second, at any limit point~$\bar\pi$, the per-state vanishing-progress
condition gives: for all $a,s$ with $\bar\pi(a|s) > 0$,
$U(a,s)(e^{\alpha\,w(s,a)\,U(s,a)}-1) = 0$.
Since $w(s,a) > 0$, this forces $U(s,a) = 0$, i.e.,
$Q^{\bar\pi}(s,a) = V^{\bar\pi}(s)$.
If this holds at every reachable state, $V^{\bar\pi}$ satisfies
the Bellman optimality equation, so $V^{\bar\pi} = V^*$ by
the contraction property.
Any sub-optimal limit point must therefore place zero mass on
some improving action.

Third, the DG logit update grows the improving action's logit at rate
$\alpha\,w(\tilde a)\,U(\tilde a) > 0$ while the corner action's
logit stalls, contradicting convergence to the sub-optimal corner
(same argument as Theorem~\ref{thm:mdp_convergence}, Step~4). Since all limit points are optimal, by uniqueness, $\pi_t \to \pi^*$.
\end{proof}

\begin{corollary}[DG $O(1/t)$ rate in MDPs]
\label{cor:dg_mdp_rate}
Under the conditions of Theorem~\ref{thm:dg_mdp_convergence},
$\delta_t := V^* - V(\pi_t) = O(1/t)$.
\end{corollary}

\begin{proof}
The proof follows Corollary~\ref{cor:mdp_rate}.
The progress bound becomes
$\delta_t - \delta_{t+1} \ge \frac{\alpha\,d_{\min}}{Z_{\max}(1-\gamma)}
\sum_s w(s,a^*(s))\,\pi_t(a^*(s)|s)\,U(s,a^*(s))^2$.
Asymptotically, $w(s,a^*(s)) \to \sigma(0) = 1/2$ and
$\pi_t(a^*(s)|s) \ge 1/2$, so
$\delta_t - \delta_{t+1} \ge c_{\mathrm{DG}}\,\delta_t^2$
by the same Cauchy--Schwarz argument, with
$c_{\mathrm{DG}} = \alpha\,d_{\min}(1{-}\gamma)/(4\,Z_{\max})$.
The reciprocal telescope gives $\delta_t = O(1/t)$.
\end{proof}

%\newpage
%\input{checklist.tex}

\end{document}